\documentclass{article}

\usepackage{microtype}
\usepackage{graphicx}
\usepackage{subfigure}
\usepackage{booktabs}
\usepackage{subcaption}
\usepackage{multicol}

\newcommand{\latinabv}[1]{\emph{#1}\,}
\newcommand{\cf}{\latinabv{cf.}}

\newcommand{\Prob}{\mathbb{P}}
\newcommand{\propone}{centroidal}
\newcommand{\proptwo}{cell-scoring}
\newcommand{\naive}{Kernel-WTA } 
\newcommand{\ours}{Voronoi-WTA }
\newcommand{\hist}{Histogram }
\newcommand{\dimY}{\mathrm{dim}(\cY)}
\newcommand{\cZV}{\mathcal{Z}_x^\mathrm{V}}
\newcommand{\cZH}{\mathcal{Z}^\mathrm{H}}
\usepackage{enumitem}
\setlist{nosep}
\newcommand{\myparagraph}[1]{\smallskip\noindent\textbf{#1}}
\usepackage{hyperref}

\usepackage[accepted]{icml2024}

\usepackage{amsmath}
\usepackage{amssymb}
\usepackage{mathtools}
\usepackage{amsthm}
\usepackage{amsfonts}% Math symbols
\usepackage[scr=rsfs]{mathalpha}
\usepackage{dsfont}
\usepackage{stmaryrd}
\usepackage{hyperref}
\usepackage{xparse}
\usepackage{graphicx}
\usepackage{afterpage}
\usepackage[capitalize,noabbrev]{cleveref}

\theoremstyle{plain}
\newtheorem{theorem}{Theorem}[section]
\newtheorem{proposition}[theorem]{Proposition}

\newtheorem*{definition}{Definition}
\newtheorem{assumption}[theorem]{Assumption}

\newcommand{\bx}{x}
\newcommand{\by}{y}
\newcommand{\bz}{z}
\newcommand{\bg}{g}
\newcommand{\bs}{s}
\newcommand{\cY}{\mathcal{Y}}
\newcommand{\cYk}{\mathcal{Y}^k}
\newcommand{\cX}{\mathcal{X}}
\newcommand{\cZ}{\mathcal{Z}}
\newcommand{\cN}{\mathcal{N}}
\newcommand{\cG}{\mathcal{G}}
\newcommand{\cL}{\mathcal{L}}
\newcommand{\cF}{\mathcal{F}}
\newcommand{\bN}{\mathbb{N}}
\newcommand{\bR}{\mathbb{R}}
\newcommand{\bI}{\mathds{1}}
\newcommand{\bP}{\mathbb{P}}

\newcommand{\hP}{\hat{\mathbb{P}}}

\newcommand{\ph}{\mathbb{P}_x^{\;\mathrm{H}}}

\newcommand{\pdwta}{\mathbb{P}_x^{\;\mathrm{D-WTA}}}
\newcommand{\puwta}{\mathbb{P}_x^{\;\mathrm{U-WTA}}} 
\newcommand{\pkwta}{\mathbb{P}_x^{\;\mathrm{K-WTA}}} 
\newcommand{\ptwta}{\mathbb{P}_x^{\;\mathrm{V-WTA}}}

\newcommand{\dx}{\mathrm{d}x}
\newcommand{\dy}{\mathrm{d}y}

\newcommand{\nhyp}{K} % number of hypotheses
\newcommand{\bn}{\llbracket1,\nhyp\rrbracket}
\newcommand{\distortion}{\mathcal{R}}

\newcommand{\Ykx}{\mathcal{Y}_\theta^k(x)} % Voronoi cell k dependency on x
\newcommand{\Ykg}{\mathcal{Y}_k(g)} % Voronoi cell k 
\newcommand{\Ykz}{\mathcal{Y}_k(z)}
\newcommand{\f}{f_{\theta}} % set of hypotheses
\newcommand{\fk}{f_{\theta}^k} % hypothesis model k
\newcommand{\fkfirst}{f_{\theta}^1} % hypothesis model k
\newcommand{\fklast}{f_{\theta}^\nhyp} % hypothesis model k
\newcommand{\gk}{\gamma_{\theta}^k} % scoring model k
\newcommand{\gkfirst}{\gamma_{\theta}^1} % scoring model k
\newcommand{\gklast}{\gamma_{\theta}^{\nhyp}} % scoring model k
\newcommand{\fkz}{z_k} % prediction of hypothesis k
\newcommand{\fkx}{f_\theta^k(x)} % prediction of hypothesis k
\newcommand{\pyx}{\rho_x(\by)} % p(y|x) 
\newcommand{\px}{\rho_x} % p(.|x)
\newcommand{\hpyx}{\hat{\rho}_x(\by)} % p_hat(y|x) 
\newcommand{\hpx}{\hat{\rho}_x} % : p_hat(.|x)
\newcommand{\zk}{z_k} % z_k
\newcommand{\zkx}{z_k(x)} % z_k(x)

\newcommand{\AYk}{\mathcal{Y}^k}  
\newcommand{\AYkn}{\mathcal{Y}^{k}_{\nhyp}}
\newcommand{\Azk}{z_k}  

\newcommand{\rkn}{r_\nhyp^k}
\newcommand{\Azkn}{z^{k}_{\nhyp}}

\usepackage[textsize=tiny]{todonotes}
\usepackage{minitoc}

\begin{document}

\twocolumn[
\icmltitle{Winner-takes-all learners are geometry-aware conditional density estimators}

\icmlsetsymbol{equal}{*}

\begin{icmlauthorlist}
\icmlauthor{Victor Letzelter}{equal,valeo,telecom}
\icmlauthor{David Perera}{equal,telecom}
\icmlauthor{Cédric Rommel}{meta}\\
\icmlauthor{Mathieu Fontaine}{telecom}
\icmlauthor{Slim Essid}{telecom}
\icmlauthor{Gaël Richard}{telecom}
\icmlauthor{Patrick Pérez}{kyutai}
\end{icmlauthorlist}

\icmlaffiliation{valeo}{Valeo.ai, Paris, France}
\icmlaffiliation{meta}{Meta AI, Paris, France}
\icmlaffiliation{telecom}{LTCI, Télécom Paris, Institut Polytechnique de Paris, France}
\icmlaffiliation{kyutai}{Kyutai, Paris, France}

\icmlcorrespondingauthor{Victor Letzelter}{victor.letzelter@telecom-paris.fr}

\icmlkeywords{Conditional density estimation, Multiple choice learning, Voronoi tessellation}

\vskip 0.3in
]

\renewcommand{\icmlEqualContribution}{\textsuperscript{*}Equal contribution.}
\printAffiliationsAndNotice{\icmlEqualContribution}

\begin{abstract}
Winner-takes-all training is a simple learning paradigm, which handles ambiguous tasks by predicting a set of plausible hypotheses. Recently, a connection was established between Winner-takes-all training and centroidal Voronoi tessellations, showing that, once trained, hypotheses should quantize optimally the shape of the conditional distribution to predict. However, the best use of these hypotheses for uncertainty quantification is still an open question.
In this work, we show how to leverage the appealing geometric properties of the Winner-takes-all learners for conditional density estimation, without modifying its original training scheme. We theoretically establish the advantages of our novel estimator both in terms of quantization and density estimation, and we demonstrate its competitiveness on synthetic and real-world datasets, including audio data.
\end{abstract}

\section{Introduction}

Machine-learning-based predictive systems are faced with a fundamental limitation when there is some ambiguity in the data or in the task itself. This results in a non-deterministic relationship between inputs and outputs, which is challenging to cope with. Characterizing this inherent uncertainty is the problem of conditional distribution estimation. 

The recently introduced Winner-takes-all (WTA) training scheme \cite{guzman2012multiple, lee2016stochastic} is a novel approach addressing ambiguity in machine learning. This scheme leverages several models, generally a neural network equipped with several heads, to produce multiple predictions, also called \textit{hypotheses}.
It trains these hypotheses competitively, updating only the hypothesis that yields the current best prediction. Experimental evidence has demonstrated that this approach enhances the diversity of predictions, with each head gradually specializing in a subset of the data distribution. 

At the same time, a limited body of work has tried to theoretically elucidate the appealing characteristics of Winner-takes-all learners. 
Specifically, \citet{rupprecht2017learning} described the geometrical properties of the trained WTA learners using the formalism of \textit{centroidal Voronoi tessellations}. This approach is linked to the field of quantization, where the objective is to represent an arbitrary distribution optimally using a finite set of points \cite{zador1982asymptotic}. 

Being able to quantize a distribution in an input-dependent manner, WTA learners have the potential to model the \textit{geometric} information of a distribution. This raises the following question: can WTA learners be used to make accurate \textit{probabilistic} predictions? This paper affirms this possibility.

We build upon the recent findings of \citet{letzelter2023resilient}, which proposed modeling uncertainty from WTA predictions, using either Dirac or uniform mixtures. We extend this idea by proposing a kernel-based density estimator for WTA predictors. This enables the computation of uncertainty metrics, such as the negative log-likelihood, from trained WTA models. This development introduces a novel method for the probabilistic evaluation of WTA predictions. Notably, it can be used even when only a single target from the conditional distribution is available for each input.

The key contributions of this work are as follows:

\begin{enumerate}
    \item We introduce an estimator that provides a comprehensive probabilistic interpretation of WTA predictions while retaining their appealing geometric properties.
    \item We mathematically validate the competitiveness of our estimator, both in terms of geometric quantization properties and probabilistic convergence, as the number of hypotheses increases.
    \item We empirically substantiate our estimator through experiments on both synthetic and real-world data, including audio signals.\footnote{Code at \href{https://github.com/Victorletzelter/VoronoiWTA}{https://github.com/Victorletzelter/VoronoiWTA}.}
\end{enumerate}

\section{Background}
\label{sec:background}

\subsection{Winner-takes-all training}

The Winner-takes-all training scheme is the basic building block of the \textit{Multiple Choice Learning} family of approaches \cite{guzman2012multiple, lee2016stochastic, lee2017confident, tian2019versatile}.
It was introduced to deal with inherently ambiguous prediction tasks. Specifically, we are not only interested in predicting a single output $f_{\theta}(\bx) \in \cY$ from a given input $\bx \in \cX$ ($f_\theta$ can typically be a deep neural network with parameters $\theta$).
Instead, we want to perform several predictions $f_{\theta}^1(\bx), \dots, f_{\theta}^\nhyp(\bx)$ accounting for $K$ potential outcomes.

More precisely, let $f_{\theta} \triangleq (\fkfirst, \dots, \fklast) \in \mathcal{F}(\cX,\cY^\nhyp)$, which could be for instance a multi-head deep neural network, and let $(\bx,\by) \in \cX \times \cY$ be 
a pair sampled from a joint distribution $\bP$ (with density $\rho(\bx,\by)$).

In a supervised setting, the WTA training consists in:
\begin{enumerate}
    \item performing a forward pass through the loss $\ell$ for all predictors,
    \item then backpropagating the loss gradients for the selected \textit{winner} hypothesis:
\begin{equation}
   \mathcal{L}^{\mathrm{WTA}}(\theta) \triangleq \min _{k \in \bn} \ell\left(\fk\left(\bx\right), \by\right)\,. \label{eq:loss}
\end{equation}
\end{enumerate}

This two-step approach, originally proposed by \citet{lee2016stochastic}, makes it possible to use gradient-based optimization, despite the non-differentiability of the $\min$ operator in \eqref{eq:loss}.

\subsection{Desirable geometrical properties}
\label{sec:geometrical-properties}

\citet{rupprecht2017learning} have shown that, once trained, the output of the set of predictors $(\fkfirst(\bx), \cdots, \fklast(\bx))$ can be interpreted as an \emph{input-dependent centroidal Voronoi tessellation}, thereby providing a geometrical probabilistic interpretation of WTA.
As done by \citet{rupprecht2017learning}, we study the case where $\ell(\hat{\by}, \by) = \|\hat{\by}-\by\|^2$ is the the $\mathrm{L}^2$ loss.
In a standard machine learning setting, where a single prediction is provided, one can prove that the risk
\begin{equation}
   \mathbb{E}_{(\bx,\by) \sim \rho(\bx,\by)}[ \ell(f_{\theta}(\bx),\by)]\,,
    \label{eq:vanilla-risk}
\end{equation}
is minimized when $\; \forall x\in\cX, \; \f(\bx) = \mathbb{E}[Y_x]$,
noting $Y_x\sim\bP_x$ the conditional distribution
and $\rho_x$ its density. The proof of this result is based on the customary assumption that the predictor $f_{\theta}$ is 
sufficiently expressive, so that minimizing the risk \eqref{eq:vanilla-risk} is equivalent to minimizing the input-dependent risk, $\mathbb{E}_{\by \sim \pyx}[ \ell(f_{\theta}(\bx),\by)]$, for each fixed input $\bx$.
When multiple predictors are used, as in the WTA case, the situation is more complex.
In this case, after defining Voronoi cells as:
\begin{equation}
    \Ykg \triangleq\left\{\by \in \mathcal{Y}\;|\;\ell\left(\bg_k, \by\right)<\ell\left(\bg_r, \by\right), \forall r \neq k\right\}\,,
\end{equation}
for some arbitrary set of generators $(\bg_1,\cdots,\bg_\nhyp)$,
the input-dependent risk writes, for each $\bx \in \cX$, as 
\begin{equation}
    \sum_{k=1}^{\nhyp} \int_{\mathcal{Y}_{\theta}^k(\bx)} \ell(f^k_{\theta}(\bx),\by) \rho_{x}(y) \mathrm{d}\by\,,
    \label{eq:wta-risk}
\end{equation}
where for ease of notation ${\mathcal{Y}^k_{\theta}(\bx) \triangleq \mathcal{Y}_k(f_{\theta}(\bx)})$.
Note that \eqref{eq:wta-risk} is not differentiable with respect to the parameters $\theta$, which are involved both in the integrand
and in the integration domain.
This issue can be alleviated by uncoupling the two variables and defining:
\begin{equation}
   \mathcal{K}(\bg,\bz) \triangleq \sum_{k=1}^\nhyp \int_{\Ykg} \ell\left(\bz_k, \by\right)\rho_{x}(y)\mathrm{d} \by\,,
    \label{problm}
    \end{equation}
where ${\bz = (\bz_1,\cdots,\bz_\nhyp)}$.
Note that \eqref{eq:wta-risk} corresponds to $\mathcal{K}(f_\theta(\bx), f_\theta(\bx))$.

For the purpose of the next proposition, let us define a centroidal Voronoi tessellation.
\begin{definition}[Centroidal Voronoi Tessellation]
    We say that $\{\Ykz\}$ forms a centroidal Voronoi tessellation with respect to a density function $\rho_x$ and a loss function $\ell$ if, for each cell $k$, the generator $\zk$ minimizes the weighted loss over its region: 
$$
\int_{\mathcal{Y}_{k}(z)} \rho_x(\by) \ell\left(z_k, \by\right) \mathrm{d}\by=\inf _{z^{\prime} \in \mathcal{Y}^{\star}_k(z)} \int_{\mathcal{Y}_k(z)} \rho_x(\by) \ell(z^{\prime}, \by)\mathrm{d}\by\,,
$$
where $\mathcal{Y}^{\star}_k(z)$ is the closure of $\mathcal{Y}_k(z)$.
\end{definition}

Based on formulation \eqref{problm}, \citet{rupprecht2017learning} show that one can adapt the following result.
\begin{proposition}[\citeauthor{du1999centroidal}, \citeyear{du1999centroidal}]
A necessary condition for minimizing \eqref{problm} is that $\Ykg$ are the Voronoi regions generated by the $\zk$, and simultaneously, $\{\Ykz\}$ forms a centroidal Voronoi tessellation generated by $\{\zk\}$.
\label{th:du_et_al}
\end{proposition}

In particular, if $\ell$ is the $L^2$-loss, this condition implies that, for each non-zero probability cell, the optimal hypotheses placements correspond to cell-restricted conditional expectations as stated in Theorem 1 of \citet{rupprecht2017learning}:
\begin{equation}
    \fk(\bx) = \mathbb{E}[Y_{\bx} \mid Y_{\bx} \in \Ykx]\,.
    \label{eq:expectation}
\end{equation}

This necessary condition, which offers a geometrical interpretation of the Winner-takes-all optimum, has been verified experimentally in previous works \cite{rupprecht2017learning, letzelter2023resilient}, thus demonstrating the method's potential to predict \textit{input-dependent} centroidal Voronoi tessellations using deep neural networks.

\subsection{Probabilistic interpretation as a mixture model}\label{sec:rMCL}
Proposition \ref{th:du_et_al} highlights the geometric advantages of WTA, but it does not provide a full probabilistic interpretation of this method.
First, \eqref{eq:expectation} is valid only in
Voronoi cells with strictly positive mass,
\textit{i.e.}, containing at least one sample from the ground-truth empirical distribution.
Furthermore, the WTA predictor from \citet{rupprecht2017learning} implicitly affects equal probability to all Voronoi cells, regardless of their ground-truth probability mass.

As a possible solution, \citet{tian2019versatile} and \citet{letzelter2023resilient}
propose to additionally train \textit{score} heads ${\gkfirst, \ldots, \gklast \in \mathcal{F}(\cX,[0,1])}$, estimating the probability mass of each cell $\Prob(Y_{\bx} \in \Ykx)$ by jointly optimizing in $\theta$ the WTA loss \eqref{eq:loss} with the cross-entropy 
\begin{equation}
\mathcal{L}^{\mathrm{scoring}}(\theta) \triangleq \sum_{k=1}^\nhyp\mathrm{BCE}
\left(\mathds{1}\left[\by\in\Ykx\right],\gamma_\theta^k(\bx)\right)\,,
\end{equation}
between the predicted assignation probability $\gk(x)$ and the actual assignation, where $\mathrm{BCE}(p, q) \triangleq - p\log(q) - (1-p)\log(1-q)$. The full training objective is therefore defined as a compound loss $\mathcal{L}^{\mathrm{WTA}}+\beta \mathcal{L}^{\mathrm{scoring}}$. Mirroring Proposition \ref{th:du_et_al}, one can show that a necessary condition to minimize the scoring objective is that each $\gk(\bx)$ is an unbiased estimator of the probability mass of its cell:
\begin{equation}
    \gk(\bx) = \Prob(Y_{\bx} \in \Ykx)\,.
    \label{eq:score}
\end{equation}

Assuming now that \eqref{eq:expectation} and \eqref{eq:score} are verified after training, through the minimization of the combined objective, \citet{letzelter2023resilient} argued that it is possible to interpret the outputs of such a model probabilistically, as a Dirac mixture:
\begin{equation}
    \label{eq:mixturedelta}
    \hpyx = \sum_{k=1}^{\nhyp} \gamma^{k}_{\theta}(\bx) \delta_{f^{k}_{\theta}(\bx)}(\by)\,.
\end{equation}

Let $\hat{Y}_{\bx} \sim \hpx$ denote the random variable sampled from this estimated conditional distribution.
The Dirac mixture interpretation \eqref{eq:mixturedelta} has at least two desired properties:
\begin{enumerate}
    \itemsep0em
    \item \textbf{[\propone\ property]} the cell-restricted expectation with respect to the estimated distribution matches the ground truth:
    \begin{equation}
        \mathbb{E}[\hat{Y}_{\bx} \mid \hat{Y}_{\bx} \in \Ykx]=\mathbb{E}[Y_{\bx} \mid Y_{\bx} \in \Ykx]\,,
        \label{eq:centroidal-prop}
    \end{equation}
    \item \textbf{[\proptwo\ property]} the predicted probability mass of the Voronoi cells is unbiased:
    \begin{equation}
        \Prob(\hat{Y}_{\bx} \in \cYk_{\theta}(\bx))=\Prob(Y_{\bx} \in \cYk_{\theta}(\bx))\,.
        \label{eq:scoring-prop}
    \end{equation}
\end{enumerate}

This interpretation is hence appealing as it captures the global shape of the distribution.
However, it presents a major caveat: it does not
capture local variations of mass within the Voronoi cells.\\

\section{Limitations of current estimators}
\label{sec:unifying}

We illustrate hereafter the main limitations of the
probabilistic interpretation of
the score-based WTA proposed in \citet{letzelter2023resilient}. To this end, we consider a toy example inspired from \citet{rupprecht2017learning}.

Our goal is to predict an input-dependent distribution ${\hpyx}$, where
$x$ lives in the unit-segment $\cX = [0,1]$, and $y$ is restricted to the 2D-square $\cY = [-1,1]^2$.
The latter is split into four quadrants:
$S_1=[-1,0) \times[-1,0), S_2= [-1,0) \times[0,1], S_3=[0,1] \times[-1,0)$ and $S_4=[0,1] \times[0,1] $.
The target distribution for each $\bx$ is then generated by first sampling one of the four quadrants with probabilities $p\left(S_1\right)=p\left(S_4\right)=\frac{1-x}{2}$ and $p\left(S_2\right)=p\left(S_3\right)=\frac{x}{2}$.
Once a region is sampled, a point is then drawn from a predetermined distribution restricted to that region:
uniform distributions in $S_1$ and $S_4$, and Gaussian distributions in $S_2$ and $S_3$ (with different standard deviations, respectively, $\sigma_2 \gg \sigma_3$). 

We trained a $20$-hypothesis scoring-based WTA model, consisting of a three-layer MLP, on this dataset. 
The predictions for three inputs $x \in \{0.01,0.6,0.9\}$ are shown in Figure \ref{fig:uni-to-gaussians}. 

For small values of $x$, the ground-truth conditional distribution is piece-wise uniform.
In this situation, the mass does not vary much inside the predicted Voronoi cells, and the whole distribution is hence well summarized by the predicted hypotheses and scores alone.
However, the same cannot be said as $x$ increases. 
Indeed, when $x \in \{0.6,0.9\}$, we see on the bottom-right quadrant ($S_3$) that the small-variance Gaussian is modeled by a single hypothesis and Voronoi cell.
Although the hypothesis seems to be well-positioned at the true Gaussian mean and its corresponding cell seems to verify both \propone\ and \proptwo\ properties, the local mass variations within the cell are not well-described.
More precisely, neither a Dirac delta, as in \eqref{eq:mixturedelta}, nor a cell-restricted uniform distribution seem like good estimations of the underlying conditional probability density within $S_3$.

This problem of intra-cell density approximation can be partially mitigated by increasing the number of hypotheses and using uniform mixtures. However, we argue that more accurate estimators can be built from the adaptive grid provided by WTA, even when the number of hypotheses is low.

This example highlights the need for
WTA-based conditional density estimators taking into account the data distribution geometry through optimal hypotheses placement.

\begin{figure}[ht]
    \centering
    \includegraphics[width=1.0\columnwidth]{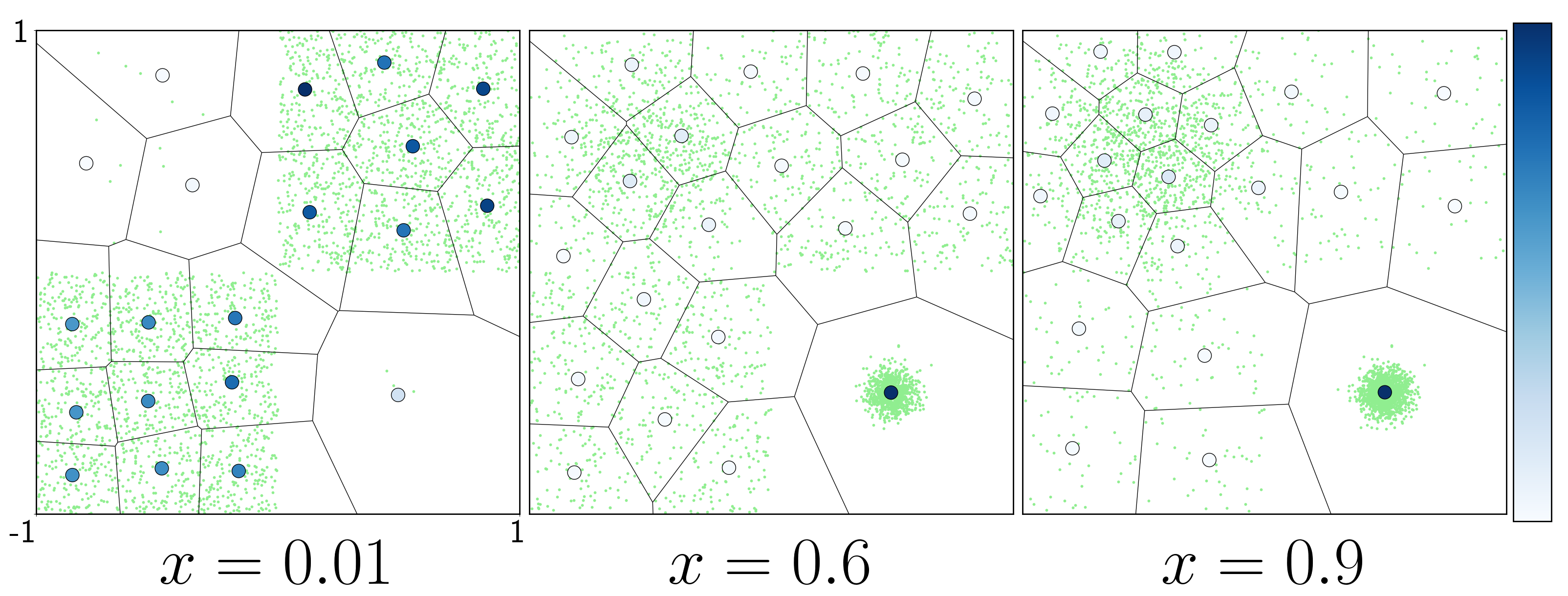}
     \caption{\textbf{Limitations of Dirac Mixtures.} Model predictions for different inputs $x$ (columns) are shown with blue-shaded circles; the colorbar indicates hypothesis scores. Green points depict the target distribution for each input. Black lines mark the boundaries of the Voronoi tessellation associated with the predictions.
    }
    \label{fig:uni-to-gaussians}
\end{figure}

\section{Conditional density approximation}
\label{sec:method}

The goal of this work is to propose a probabilistic interpretation of the Winner-takes-all predictions as a conditional density estimator that preserves the global geometric properties of the predictions (centroidal Voronoi tessellation) and captures the local variations of the probability density, including inside Voronoi cells. Ideally, we would also like our estimators to verify both centroidal \eqref{eq:centroidal-prop} and cell-scoring properties \eqref{eq:scoring-prop}.

\subsection{Kernel WTA}
\label{sec:Kernel-WTA}
A straightforward way to model intra-cell density variations is
to place a kernel $K_{h}(\cdot, \cdot): \cY \times \cY \rightarrow \mathbb{R}_+$ on each hypothesis $k$ as in a traditional Parzen estimator ~\cite{rosenblatt1956remarks,parzen1962estimation}.
This procedure defines the following conditional density estimator, called \mbox{\emph{Kernel-WTA}} hereafter:
\begin{equation}
    \hpyx = \sum_{k=1}^{\nhyp} \gk(\bx) K_h(\fk(\bx),\by)\,,
    \label{eq:kde}
\end{equation}
where $h \in \mathbb{R}_{+}^{*}$ is the scaling factor of the kernel. 
In this work, we consider only isotropic kernels, \textit{e.g.,}
Gaussian, assumed to integrate to $1$ in their second variable. 
Despite being simple and allowing intra-cell variations to be modeled, this method has drawbacks.
Indeed, we see from \eqref{eq:kde} that, whenever $h$ is too large, the kernels begin to diffuse density out of their Voronoi cells.
As a result, \emph{neither} the \propone\ \eqref{eq:centroidal-prop} nor the \proptwo\ \eqref{eq:scoring-prop} properties hold anymore, meaning that the geometric advantages offered by the Winner-takes-all predictions are not fully preserved.
As a second drawback, the convergence of Kernel-WTA is highly dependent on the choice of $h$, as discussed in Section \ref{sec:theory}.

\subsection{Voronoi WTA}
\label{sec:truncated-kernels}
As mentioned in the previous section, the straightforward Kernel-WTA fails to preserve the geometric properties when $h$ is too large.
Inspired by \citet{polianskii2022voronoi}, we propose to alleviate this problem using truncated kernels.
More precisely, let $V(g_k, K_h) \triangleq \int_{\cY_k(g)} K_h \left(g_k, \tilde{\by}\right) \mathrm{d} \tilde{\by}$ be the volume of the Voronoi cell defined by generator $g_k$, under the metric induced by kernel $K_h$.
We define the \mbox{\emph{Voronoi-WTA}} estimator as:
\begin{equation}
    \hpyx=\sum_{k=1}^\nhyp \gamma_\theta^k(\bx) \frac{K_h\left(f_\theta^k(\bx), \by\right)}{V(f_\theta^k(\bx), K_h)} \mathds{1}\left(\by \in \cYk_\theta(\bx)\right)\,.
    \label{eq:voronoi-wta}
\end{equation}

Unlike Kernel-WTA, the density estimations derived from \eqref{eq:voronoi-wta} are designed to fulfill the \proptwo\ property, \textit{i.e.}, if ${\hat{Y}_{x} \sim \hpx}$, $\mathbb{P}(\hat{Y}_{x} \in \Ykx) = \gk(\bx) = \mathbb{P}(Y_{\bx} \in \Ykx)$. Using this property, the convergence in distribution of Voronoi-WTA as $K$ approaches infinity, shown in Section \ref{sec:convergence}, is independent of the choice of $h$. As in \eqref{eq:kde}, $h$ remains constant by design for each input $x$, and in each cell $k$. Note that increasing $h$ causes a slight shift in the barycenter of the predicted distribution from the ground-truth expectation in each cell. Nonetheless, in our setup, $h$ will be optimized after the optimization of $\theta$, ensuring that \eqref{eq:expectation} is still verified.

\subsection{Likelihood computation and sampling}
\label{sec:likelihood}

In practice, the use of the Voronoi-WTA defined in \eqref{eq:voronoi-wta} raises three main questions: (1) how to sample from this estimator, (2) how to compute likelihoods, and (3) how to choose the scaling factor $h$.

\myparagraph{Sampling.} \emph{Rejection sampling} is a simple way of sampling from \eqref{eq:voronoi-wta}.
In practice, one can first draw a Voronoi cell $k \in \bn$ from the discrete distribution of predicted scores $\{\gamma^l_\theta(x)\}$, then sample from the kernel $K_h(f^k_\theta(\bx),\cdot)$ until a sample falls in cell $k$. This approach was efficient enough for all our experiments. Note that whenever the number of hypotheses or the dimension is large, the more efficient \emph{hit-and-run} sampling method from \citet{polianskii2022voronoi} may be adapted for this setup.

\myparagraph{Likelihood computation.}
The difficulty in computing \eqref{eq:voronoi-wta} comes from the normalization term $V(\fk(\bx), K_h)$.
In practice, it can be computed efficiently by re-writing $V(\fk(\bx), K_h)$ as a double-integral in spherical coordinates as explained in \citet{polianskii2022voronoi}.
As the inner integral often has a closed-form solution for usual kernels, a simple Monte Carlo approximation of the outer integral allows us to efficiently estimate $V(\fk(\bx), K_h)$.

\myparagraph{Choosing the scaling factor.}
We limit ourselves in this work to the case where the kernels in all Voronoi cells share the same scaling factor $h$.
The latter has to be tuned, which can be done in practice based on the likelihood obtained by the model
in a validation set.
Note that when $h \rightarrow 0$,
both Kernel-WTA  \eqref{eq:kde} and Voronoi-WTA estimators \eqref{eq:voronoi-wta} are equivalent to the Dirac mixture \eqref{eq:mixturedelta} from \citet{letzelter2023resilient}.
However, when $h$ increases, Kernel-WTA loses the geometry captured by the hypotheses $\{\fk(\bx)\}$, while Voronoi-WTA preserves it, converging to a piece-wise uniform distribution defined on the Voronoi tessellation.
This is verified experimentally in Section \ref{sec:experiments}.

\section{Theoretical properties}
\label{sec:theory}

In this section, we present informally the two main theoretical results of this work.
These propositions are made precise in the appendix, together with other complementary results and their corresponding proofs.

\subsection{Convergence in distribution independent of $h$}
\label{sec:convergence}

As a first theoretical contribution, we show in the following proposition that Voronoi-WTA is an effective conditional density estimator, in the sense that it converges to the ground-truth underlying distribution:
\begin{proposition} \label{th:convergence-paper}
    Under mild assumptions on $\cY$ and the data distribution, Voronoi-WTA (seen as a density estimator) converges in probability towards the conditional distribution $\bP_{\bx}$ when the number \nhyp~of hypotheses grows to infinity.
\end{proposition}
A formal version of this result can be found in the appendix (Theorem \ref{th:convergence}).

This result is similar to Theorem 4.1 by \citet{polianskii2022voronoi} on the convergence of the Compactified Voronoi Density Estimator.
Note, however, that our setting differs on two main points:
\begin{enumerate}
    \item we study the more general problem of \emph{conditional} densities,
    \item our cell centroids correspond to the hypotheses predicted by a WTA estimator $\fkz=\fkx$, while \citet{polianskii2022voronoi} assume random generators $\fkz \sim \px$.
\end{enumerate}

To deal with the last point, we assume that the underlying WTA estimator predicting $\{\fk(\bx), \gk(\bx)\}$ has converged towards a global minimum of its WTA and scoring objectives.
We give a sketch of the proof of this result below. 

\textit{Proof.} The convergence relies on a useful property of the Voronoi cells obtained when we minimize the WTA objective \eqref{eq:loss}:
    their diameter vanishes as the number $\nhyp$ of hypotheses grows to infinity. This observation is then used to prove the convergence.

    Let $\bP_\nhyp$ be the trained Voronoi-WTA estimator with $\nhyp$ hypotheses. By the Portmanteau Lemma \cite{van2000asymptotic}, it is sufficient to show that $\mathbb{P}_K(E) \to {\bP}(E)$ as $K \to +\infty$
    for any measurable set $E \subseteq \cY$ such that $\lambda(\partial E) = 0$, where $\lambda$ denotes the Lebesgue measure.
    If we fix $K$, the Voronoi tiling $(\cY_\theta^k)_{k\in\llbracket1,K\rrbracket}$ induces a partition of $E$. Accordingly, 
    we split $E$ as the disjoint union $E=E^{\text{int}}\cup E^{\text{ext}}$, where $E^{\text{int}}$ denotes the Voronoi cells included in $E$, and $E^{\text{ext}}$ the Voronoi cells intersecting its border $\partial E$.
    Now, since the radius of each cell $\cY_\theta^k$ asymptotically vanishes (Proposition \ref{th:radius}), $E^{\text{ext}}$ is concentrated on the border $\partial E$ and is therefore negligible. The proof is concluded by observing that $\mathbb{P}_K$ and $\mathbb{P}$ coincide on $E^{\text{int}}$ (\proptwo\ property \eqref{eq:scoring-prop}).

Note that this last argument does not hold for Kernel-WTA, as the \proptwo\ property does not hold for this model. Also, note that Proposition \ref{th:convergence-paper} requires no assumptions on the choice of the scaling factor $h$.
This is another advantage of Voronoi-WTA, this time in terms of uncertainty modeling. 

\subsection{Better asymptotic quantization}
\label{sec:quantization}

Voronoi-WTA estimators model the conditional distribution using an adaptive grid. To measure how well a finite set of points $\cZ=\{\fkz\}_{k\in\bn}$ approximates a data 
distribution $\bP_x$, it is customary to use the \emph{quantization error}, also called quadratic risk or quadratic distortion
\cite{pages2003optimal}:
\begin{align}
    \distortion(\mathcal{Z}) = \int_\cY \min_{z\in\cZ} \|y-z\|_2^2 \;\rho_x(\by) \mathrm{d}\by\,.
    \label{eq:quantization_error}
\end{align}
A natural baseline that we can use to evaluate the advantage of the adaptative grid provided by Voronoi-WTA is the regular grid, which we call
\emph{Histogram} hereafter
(\textit{e.g.,} \citet{imani2018improving}).
Note that the regular grid is a particular case of a Voronoi tessellation.

The quantization error is notoriously hard to study in the general case
\cite{graf2007foundations}.
However, things become amenable to analysis in the asymptotic regime.
With this in mind, Zador's theorem \cite{zador1982asymptotic}, a powerful result from quantization theory, can be used to describe the asymptotic optimal quantization error.
We sum up our observations in the following statement (see Propositions \ref{th:wta_risk} and \ref{th:histogram_risk} 
for a complete formulation).

\begin{proposition}\label{th:wta_histogram_risk}
    Under mild regularity assumptions, denoting $d=\dimY$, $J_d$ a constant depending only on the dimension, $\mathrm{vol}(\cY)$ the volume of $\cY$, and
    ${\cZV=\{\fkx\}_{k\in\bn}}$,
    the quantization error has the following asymptotic equivalent as $\nhyp \rightarrow +\infty$:
    \begin{equation}\label{eq:zador_risk}
        \distortion(\cZV) \sim J_d \; \left(\int_\cY \pyx^\frac{d}{d+2} \dy\right)^\frac{d+2}{d} \frac{1}{\nhyp^{2/d}}\,.
    \end{equation}
    Denoting $\cZH$ the fixed grid points defining the Histogram baseline, we also have
    \begin{equation}\label{eq:histogram_risk}
        \distortion(\cZH) \sim \frac{d}{12}\frac{\mathrm{vol}(\cY)^{2/d}}{K^{2/d}}\,.
    \end{equation}
\end{proposition}

Note that, in the first order, the quantization error of the Histogram baseline only depends on the volume of the support of $\bP_x$, whereas Voronoi-WTA takes into account
local density information provided by $\rho_x$. This gives an insight into how the adaptative grid underlying Voronoi-WTA fits the geometry of the data distribution. 

Furthermore, one can also observe that Voronoi-WTA and the Histogram have the same asymptotic rate of convergence, differing only by the leading constant.
However, it can be proved that this constant is smaller for Voronoi-WTA than for the Histogram baseline in most cases (Proposition \ref{th:leading_constant}).
Therefore, although the gap between 
$\distortion(\cZV)$ and $\distortion(\cZH)$ closes
as the number $\nhyp$ of hypotheses increases, Voronoi-WTA \emph{always} has a strictly better quantization error, even asymptotically. This constitutes a real advantage of the adaptive grid provided by Voronoi-WTA over a static one and was empirically verified in Section \ref{sec:experiments}.

\section{Empirical study}
\label{sec:experiments}

\begin{figure*}[t]
    \includegraphics[width=1.\linewidth]{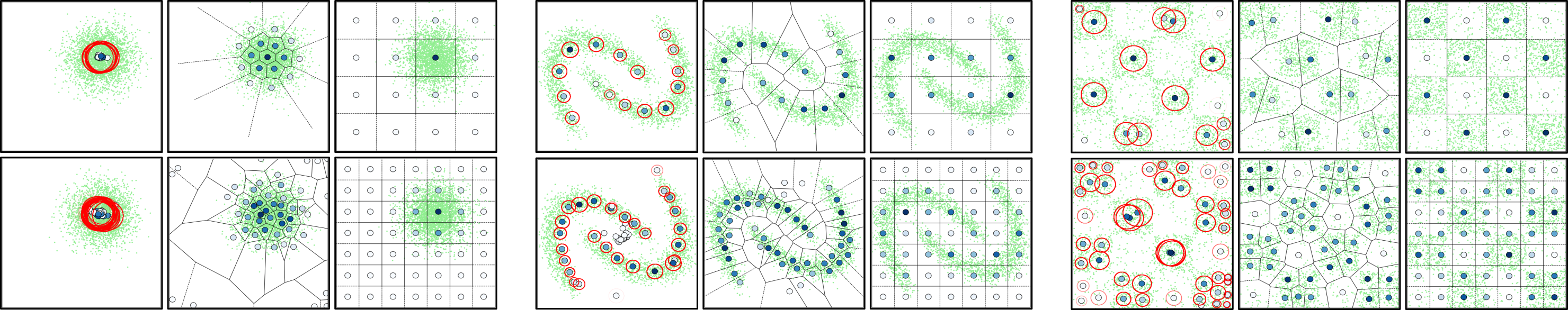}
    \caption{
    \textbf{Qualitative results.}
    Each panel shows a different dataset:
    Single Gaussian, Rotated Two Moons, and Changing Damier.
    Within each panel, columns correspond to predictions made by: MDN, Score-based WTA, and Histogram (left to right).
    Dots represent predicted (or fixed) hypotheses: means, centroids, and bins.
    Their colors encode the predicted score or mixture weight for MDN, where darker blue corresponds to higher scores.
    Red circles represent the MDN's predicted variance for each Gaussian (opacity reflects mixture weight), while WTA figures depict the Voronoi tessellations for predicted hypotheses. 
     \textit{1st row}: 16 hypotheses, \textit{2nd row}: 49 hypotheses.}
    \label{fig:qualitative}
\end{figure*}

The aim of this section is two-fold. First, it empirically justifies the relevance of the Voronoi-WTA-based conditional density estimators, compared to other possible designs based on WTA learners, such as Kernel-WTA. Second, it empirically validates the advantages of the Winner-takes-all training scheme against traditional baselines for conditional density estimation.

\subsection{Experimental setting} \label{sec:exp-setting}

We detail below our experimental settings. A more extensive description of design choices is deferred to Appendix \ref{sec:experimental_details}.

\textbf{Datasets}. We conducted experiments on four synthetic datasets, with $\cX = [0,1]$ and  $\cY = [-1,1]^2$, as well as on the UCI benchmark \cite{hernandez2015probabilistic}.
\begin{itemize}
    \item \emph{Single Gaussian} corresponds to a single, isotropic, non-centered two-dimensional Gaussian which does not move as $x$ varies.
\item \emph{Rotating Two Moons} is based on the two-moon dataset from \textsc{Scikit-Learn}~\cite{scikit-learn}, corresponding to entangled non-convex shapes.
    The target distribution was generated by rotating the latter with an angle $2\pi x$ for each $x \in [0,1]$.
\item \emph{Changing Damier} is an adaptation of the dataset proposed in \citet{rupprecht2017learning}.
    It corresponds to a checkerboard of 16 squares, gradually interpolated towards its complementary checkerboard as $x$ increases. 
\item \emph{Uniform to Gaussians} is the illustrative dataset presented in Section \ref{sec:unifying}.
\item \emph{UCI Regression datasets} \cite{ucidataset} are a standard benchmark \cite{hernandez2015probabilistic} to evaluate conditional density estimators.
\end{itemize}
 
\textbf{WTA training framework}. We used the WTA training scheme with scoring heads from Section \ref{sec:rMCL}. The density estimation was performed following the methodology described in Section \ref{sec:method}, with uniform kernels and Gaussian kernels with several scales $h$.

\textbf{Baselines}. Two standard conditional density estimation baselines were considered in our experiments: Mixture Density Networks (MDN)~\cite{bishop1994mixture} and the Histogram \cite{imani2018improving} mentioned in Section \ref{sec:quantization}. More details are given in Appendix \ref{sec:syn-baselines}.

\textbf{Architecture and training details}. In each training setup with synthetic data, we used a three-layer MLP, with 256 hidden units. The Adam optimizer~\cite{kingma2014adam} was used, and the models were trained until convergence of the training loss, using early stopping on the validation loss.

\textbf{Metrics}. To evaluate the performance of each model, we employed the Negative Log-Likelihood (NLL) and, when the target distribution is known, the Earth Mover's Distance (EMD). To assess how well each model preserved the geometry of the data distribution, we used the quantization risk, as defined in \eqref{eq:quantization_error}.

\subsection{Qualitative analysis}

Qualitative results are provided in Figure \ref{fig:qualitative}, where the predictions of Score-based WTA, Histogram, and MDN are compared. `Score-based WTA' refers to both Voronoi-WTA and Kernel-WTA, which share the same predicted hypotheses and scores represented in the figure.
Different behaviors can be observed for each of the three methods.
For instance, in the Gaussian case, MDN predictions of mixture means collapse into a single point.
This well-known \emph{mode collapse} problem~\cite{ hjorth1999regularisation, graves2013generating, rupprecht2017learning, messaoud2018structural, cui2019multimodal} is also observed for Rotated Two Moons and Changing Damier when the number of hypotheses increases.
Concerning Histogram, we see on all datasets, except the Changing Damier, that it requires more hypotheses to reach the same resolution as the score-based WTA, which is able to optimally quantize the shape of all distributions with their predicted hypotheses.

\subsection{Quantitative analysis on synthetic datasets}
\label{sec:quantitative-results}
\begin{figure*}[t]
\centering
    \includegraphics[width=1.0\linewidth]{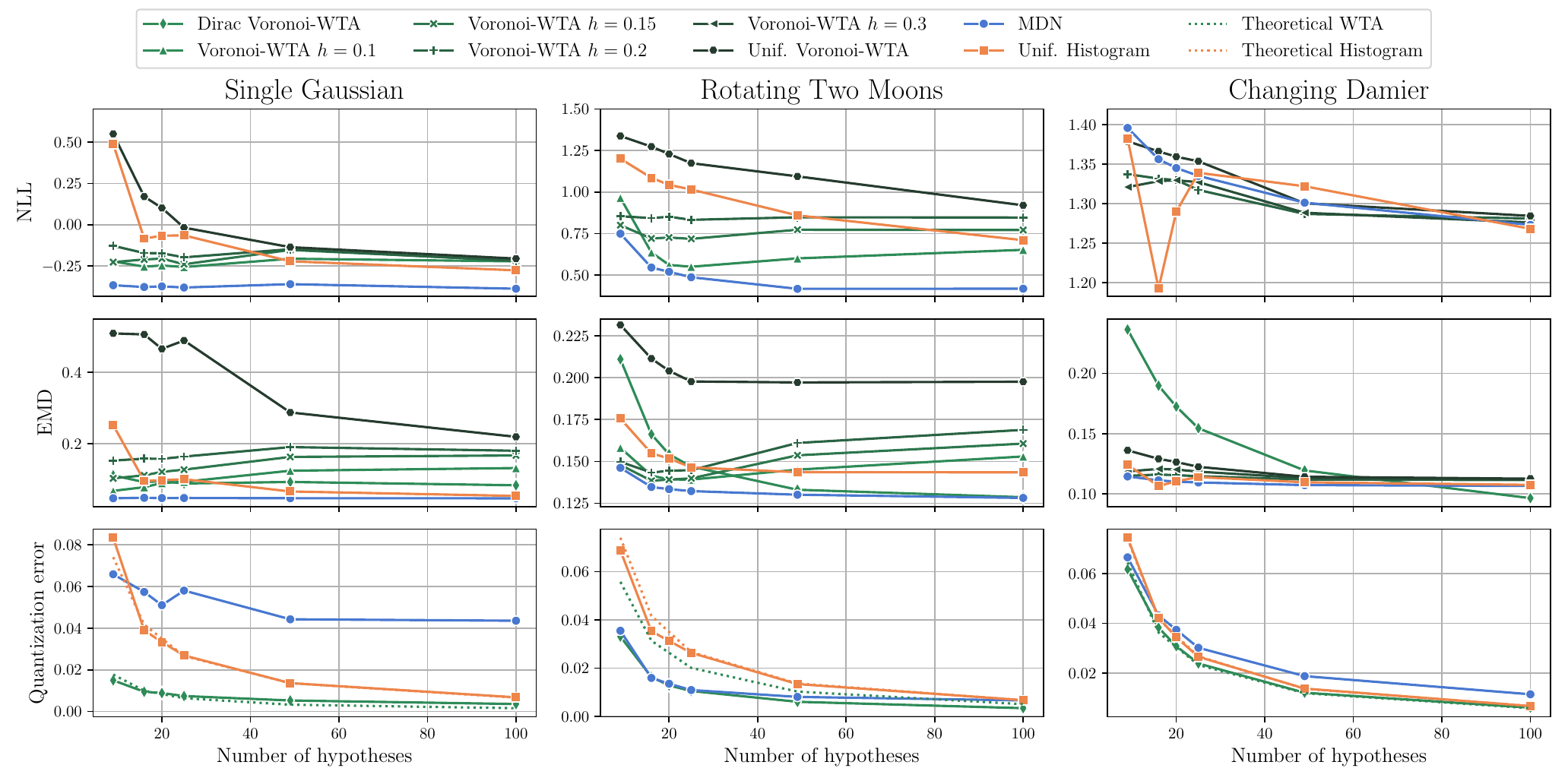}
    \caption{
    \textbf{Quantitative comparison.}
    Each column corresponds to a dataset,
    and each row to a different metric detailed in Section \ref{sec:exp-setting}.
    Dotted lines correspond to theoretical quantization errors from Proposition \ref{th:wta_histogram_risk}.
    Dirac Voronoi-WTA corresponds to the limit when the scaling factor $h \to 0$ \eqref{eq:mixturedelta}, while Unif. Voronoi-WTA is the limit when $h \to \infty$. Results are averaged over three random seeds, with standard deviations given in Appendix, Figure \ref{fig:stds_results}. Detailed discussion is given in Section \ref{sec:experiments}.
    }
    \label{fig:quantitative}
\end{figure*}

Our main quantitative results on the synthetic datasets are depicted in Figure \ref{fig:quantitative}.

\textbf{Comparison to Histogram.} 
One can see in Figure \ref{fig:quantitative} that Histogram generally does not lead to competitive performance with respect to any metric unless a high number of hypotheses is used.
An exception is observed in the case of the Changing Damier dataset where, by design, the Histogram aligns perfectly with the data when it is set to exactly 16 hypotheses (\cf Figure \ref{fig:qualitative} right). When the number of hypotheses is large enough,
the grid is sufficiently fine to represent the distribution geometry and Histogram's performance strongly improves. Note for the quantization error, however, that Histogram always performs worse than Voronoi-WTA, regardless of the number of hypotheses, as predicted by Proposition \ref{th:wta_histogram_risk}.
These results showcase quantitatively the clear advantage of Voronoi-WTA's adaptive grid over Histogram.

\textbf{Comparison to Mixture Density Networks.} MDNs have a bias towards fitting Gaussian distributions, 
and they are trained by minimizing the NLL loss. We observe, as expected, excellent NLL results, especially in the case of the Single Gaussian (Figure \ref{fig:quantitative} top). Note that MDN is the only method for which variable scaling factors are authorized in each hypothesis, giving it an immediate advantage.
Nevertheless, Voronoi-WTA still achieves on par 
performance in terms of EMD and NLL on non-Gaussian datasets, as long as the scaling parameter $h$ is well-tuned.
Furthermore, as MDNs do not benefit from the optimal quantization properties of the WTA-based method, they tend to obtain suboptimal quantization errors in most cases (Figure \ref{fig:quantitative} bottom).

\textbf{Comparison with Kernel-WTA.} We validate here the choice of Voronoi-WTA instead of the more straightforward Kernel-WTA.
Figure \ref{fig:NLLvsh} provides a comparison of both methods in the case of 16 hypotheses, in terms of NLL test performance as a function of the scaling factor $h$.
We notice the expected behavior: in the low $h$-value regime, Voronoi-WTA and Kernel-WTA curves coincide, but as $h$ increases, Voronoi-WTA's performance stabilizes, while Kernel-WTA's diverges.
Results using truncated uniform kernels are also plotted in dashed lines.
As expected, Voronoi-WTA's performance converges to the latter's as $h \to \infty$.

\textbf{Validation of Proposition \ref{th:wta_histogram_risk}}.
We plot at the bottom of Figure \ref{fig:quantitative} both theoretical quantization errors for Voronoi-WTA and Histogram derived in Proposition \ref{th:wta_histogram_risk}. First, we can notice that there is a good match between the theoretical errors and the empirical ones for all considered datasets.
This is especially true in the asymptotic regime, where the theoretical formula becomes more accurate. This validates our assumption that our underlying score-based WTA models are close to the global minimum of their training objectives.

\begin{figure}[th!]
    \centering 
    \includegraphics[width=\linewidth]{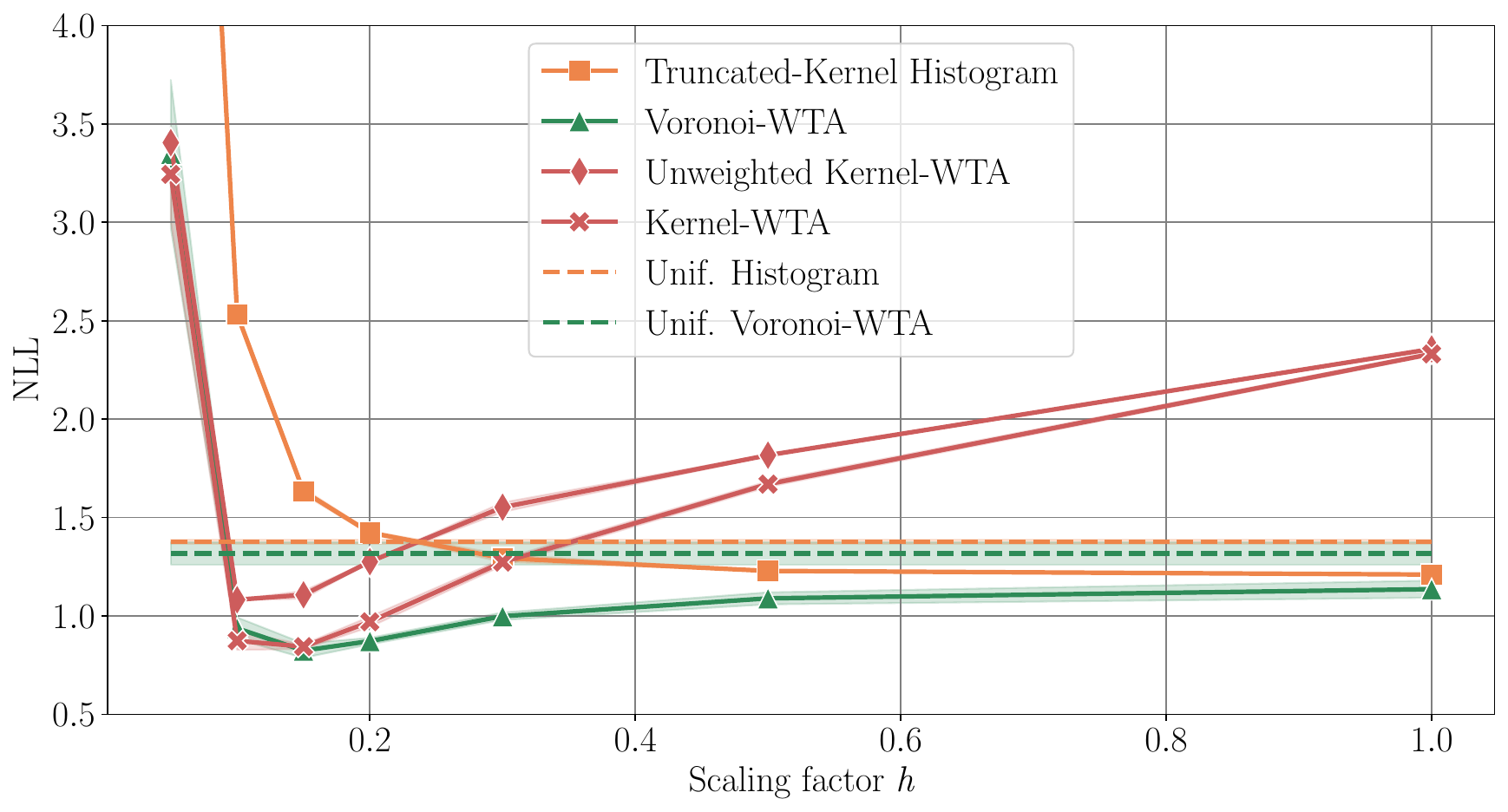}
    \caption{
    \textbf{Impact of the scaling factor.}
    Results on the dataset \textit{Uniform to Gaussians} with $16$ hypotheses, computed over three random seeds. Unweighted Kernel-WTA corresponds to \eqref{eq:kde} with fixed uniform scores $\gk(x)=1/ \nhyp$.
    Truncated-Kernel Histogram is the standard Histogram where truncated kernels are placed on the fixed hypotheses, instead of uniform kernels (Unif. Histogram) used in Figure \ref{fig:quantitative}. 
    See Appendix \ref{app:additional_results_synth} for more results.
    }
    \label{fig:NLLvsh}
\end{figure}

\subsection{Evaluation on UCI Regression Datasets}

In Table \ref{tabmain:uci-nll}, we present additional results for the UCI datasets, adhering to the experimental protocols followed by \citet{hernandez2015probabilistic, lakshminarayanan2017simple}. A more comprehensive analysis of these results is deferred to Appendix \ref{app:uci-datasets}. This appendix includes an extended version of Table \ref{tabmain:uci-nll} and also covers results using the RMSE metric (Table \ref{tab:uci-rmse}). Note that $\mathrm{dim}(\mathcal{Y}) = 1$ here.

In these datasets, the scaling factor $h$ of WTA-based models was optimized using a golden section search \cite{kiefer1953sequential}, based on the average NLL over the validation set. 
Here, this optimization was costly because it was carried out very precisely. As a result, the superior sensitivity of Kernel-WTA to the choice of $h$, when compared to Voronoi-WTA, is not expected to be visible in these results (see the optimized NLL of Voronoi-WTA and Kernel-WTA in Figure \ref{fig:NLLvsh}), particularly as $K$ is small ($K = 5$). Future research will explore how the optimization of $h$ at validation time may lead to a performance disparity between Voronoi-WTA and Kernel-WTA, especially in the context of a distribution shift between validation and test samples. 

These results further highlight the competitiveness of WTA-based density estimators in terms of NLL against Mixture Density Networks (MDN) and Deep ensembles \cite{lakshminarayanan2017simple}, especially when the dataset size is large (\textit{e.g.,} Protein, Year, \cf Table \ref{tab:datasets-description}). This finding is particularly promising given the inherent advantages of the other baselines: indeed, NLL is not directly optimized during training in WTA-based methods. Moreover, we faced stability issues when training MDN (\textit{e.g.,} numerical overflows in log-likelihood computation), that we did not encounter with Voronoi-WTA. We hope that these results will encourage further research into the properties of these estimators.

\begin{table}[h]
    \centering
    \caption{\textbf{UCI regression benchmark datasets comparing NLL with 5 hypotheses.} $^{\star}$ corresponds to reported results from \citet{lakshminarayanan2017simple}. `--' corresponds to cases where MDN has not converged. Best results are in \textbf{bold}. $\pm$ represents the standard deviation over the official splits. 
    }
    \resizebox{\columnwidth}{!}{
    \begin{tabular}{l c|ccc}
    \toprule
     & \multicolumn{4}{c}{NLL ($\downarrow$)}\\
     \cmidrule(l{2pt}r{2pt}){2-5}
    Datasets & Deep Ensembles$^{\star}$ & MDN & K-WTA & V-WTA\\
    \midrule
    Boston & \textbf{2.41 $\pm$ 0.25} & 2.95 $\pm$ 0.31 & \textbf{2.48 $\pm$ 0.16} & \textbf{2.48 $\pm$ 0.19} \\
Concrete & \textbf{3.06 $\pm$ 0.18} & 3.96 $\pm$ 0.24 & \textbf{3.09 $\pm$ 0.10} & \textbf{3.08 $\pm$ 0.12} \\
Energy & \textbf{1.38 $\pm$ 0.22} & \textbf{1.25 $\pm$ 0.25} & 2.27 $\pm$ 1.22 & 2.22 $\pm$ 1.20 \\
Kin8nm & \textbf{-1.20 $\pm$ 0.02} & -0.87 $\pm$ 0.05 & -0.73 $\pm$ 0.03 & -0.85 $\pm$ 0.05 \\
Naval & \textbf{-5.63 $\pm$ 0.05} & \textbf{-5.47 $\pm$ 0.29}  & -1.94 $\pm$ 0.00 & -3.52 $\pm$ 0.38 \\
Power & \textbf{2.79 $\pm$ 0.04} & 3.02 $\pm$ 0.07 & \textbf{2.81 $\pm$ 0.05} & \textbf{2.85 $\pm$ 0.06}\\
Protein & 2.83 $\pm$ 0.02 & -- & \textbf{2.39 $\pm$ 0.03} & \textbf{2.42 $\pm$ 0.04} \\
Wine & 0.94 $\pm$ 0.12 & \textbf{-1.53 $\pm$ 0.76} & 0.42 $\pm$ 0.18 & 0.37 $\pm$ 0.17 \\
Yacht & \textbf{1.18 $\pm$ 0.21} & 2.43 $\pm$ 0.72 & 2.23 $\pm$ 0.52 & 2.05 $\pm$ 0.46 \\
Year & 3.35 $\pm$ NA & -- & \textbf{3.26 $\pm$ NA} & 3.29 $\pm$ NA \\

\bottomrule
    \end{tabular}
    }
    \label{tabmain:uci-nll}
    \end{table}

\subsection{Experimental validation with audio data}
\label{sec:audio}
In this section, we experimentally validate our method on a real-world application, namely on the task of Sound Event Localization (SEL) \cite{adavanne2018direction, grumiaux2022survey} which involves angular localization of sound sources from input audio signals.
This task is intrinsically ambiguous, as there is spatial dispersion in the position of the sound sources to predict, either due to the sound source nature or to label noise.
We considered the audio dataset \href{https://zenodo.org/records/1237703}{ANSYN} \cite{adavanne2018sound}. This dataset is generated using simulated room impulse responses so that the position of the sound sources can be considered free of noise. We injected input-dependent label noise in the source position, conditionally to the class of the input. 
Models are trained using this noisy data to capture this input-dependent uncertainty.   

This task is challenging because it involves real-world data, label noise, and also because of the spherical geometry of the output space which calls for specific angular metrics, that are not Euclidean. Therefore, we depart from the previous theoretical and experimental setting. 
All models are trained using the same setup to ensure fair comparisons. We use the same evaluation framework introduced for the synthetic datasets (\cf Appendix \ref{app:audio} for further details).

We compare in Table \ref{tab:cvde_vs_kde} the performance of Voronoi-WTA to Kernel-WTA using isotropic von Mises-Fisher kernels with scaling factor $h \in \{0.3, 0.5, 1.0\}$, as well as to the Histogram baseline with uniform kernels.
Overall, we see that the performance of all methods tends to improve as the number of hypotheses increases, both in terms of NLL and Quantization Error. 
The competitive advantage of Voronoi-WTA against Kernel-WTA is confirmed, especially for large $h$ values. 
We also notice that this performance gap narrows with fewer hypotheses: in a low-hypothesis regime, the Voronoi tessellation resolution is smaller, reducing the possibility of capturing local geometry through truncated kernels.

\begin{table}
\renewcommand{\arraystretch}{1.4}
\fontsize{20pt}{20pt}\selectfont
\caption{
\textbf{NLL comparison of Voronoi-WTA (V-WTA) vs. Kernel-WTA (K-WTA) on audio data.} `Hist' corresponds to the histogram with a uniform kernel as a baseline. `Distortion' is the quantization error scaled up by $10^{2}$. See Section \ref{sec:audio} for the discussion and Appendix \ref{app:additional-results-appendix} for extended analysis.}
\vskip 0.05in
\centering 
\resizebox{\columnwidth}{!}{%
\begin{tabular}{lcccccccccccc}
\toprule \multicolumn{1}{c}{} & \multicolumn{9}{c}{NLL} & \multicolumn{3}{c}{Distortion} \\
\cmidrule(lr){1-1} \cmidrule(lr){2-10} \cmidrule(lr){11-13}
 $h$ &  \multicolumn{3}{c}{0.3}  &\multicolumn{3}{c}{0.5}  & \multicolumn{3}{c}{1.0} & \multicolumn{3}{c}{$\emptyset$} \\
 \cmidrule(lr){1-1}\cmidrule(lr){2-4}\cmidrule(lr){5-7}\cmidrule(lr){8-10}\cmidrule(lr){11-13}
 $K$ & 9 & 16 & 25 & 9 & 16 & 25 & 9 & 16 & 25 & 9 & 16 & 25 \\
 \midrule
V-WTA & 1.27 & \textbf{1.18} & \textbf{1.15} & \textbf{1.36} & \textbf{1.24} & \textbf{1.18} & \textbf{1.57} & \textbf{1.33} & \textbf{1.22} & \textbf{0.42} & \textbf{0.26} & \textbf{0.17} \\
K-WTA & \textbf{1.26} & 1.24 & 1.23 & 1.50 & 1.51 & 1.51 & 2.08 & 2.09 & 2.09 & \textbf{0.42} & \textbf{0.26} & \textbf{0.17} \\
 Hist & 1.72 & 1.52 & 1.45 & 1.72 & 1.52 & 1.45 & 1.72 & 1.52 & 1.45 & 1.23 & 0.72 & 0.47 \\
\bottomrule
\end{tabular}
}
\label{tab:cvde_vs_kde}
\end{table}

\section{Related work}

\noindent

\textbf{Conditional density estimation}. 
Density estimation can be tackled using
parametric methods (\textit{e.g.,} Gaussian Mixture Models) or non parametric methods (\textit{e.g.,} Kernel Density Estimation \cite{rosenblatt1956remarks}, Histograms). 
Mixture Density Networks \citet{bishop1994mixture} is a standard deep learning extension of
Gaussian Mixtures to the case of conditional densities. Their strong performances across various tasks led them to become
quite popular \cite{zen2014deep,li2019generating}. However, MDNs notoriously suffer from numerical instabilities \cite{makansi2019overcoming}, mode collapse \cite{brando2017mixture}, low contribution to the gradient of points with high predictive variance \cite{seitzer2022on} and large biases depending on
the choice of kernel \cite{polianskii2022voronoi}.
In contrast, Histogram is perhaps the simplest non-parametric alternative but is impractical in high-dimensional settings.

\noindent
\textbf{Multiple Choice  Learning.} First introduced by \citet{guzman2012multiple}, and adapted to deep learning by \citet{lee2016stochastic}, MCL is effective in various applications, notably in computer vision \cite{tian2019versatile, garcia2021distillation}.
It suffers from two main drawbacks: hypotheses collapse and overconfidence. Solutions for the first problem include top-$n$ update rules \cite{makansi2019overcoming} or allowing a small amount of gradient flow to all hypotheses \cite{rupprecht2017learning}. The second problem has been solved by the introduction of scoring models \cite{lee2017confident}. This approach has recently allowed for a probabilistic view of MCL \cite{letzelter2023resilient}. However, MCL predictions are discrete by design. One purpose of the current work is to extend MCL to density estimation, \textit{e.g.,} for improving the evaluation of such models.

\noindent

\textbf{Geometry of the Voronoi tessellations.}
Centroidal Voronoi tessellations \cite{lloyd1982least} are widely used for clustering, vector quantization \cite{gersho1979asymptotically}, and shape approximation \cite{du2003constrained}. Its popularity stems from training stability, and theoretical properties that have been extensively studied \cite{du1999centroidal}, especially in the context of optimal quantization \cite{zador1982asymptotic}. This method has been used to build continuous density estimators based on uniform \cite{okabe2009spatial} or Gaussian \cite{polianskii2022voronoi} distributions. The additional challenges raised by this 
continuous extension, such as volume estimation in high dimensional settings, or density discontinuity at the cell boundaries, have been discussed in the literature~\cite{polianskii2022voronoi,marchetti2023efficient}. However, none of these methods have yet been extended to the conditional setting, which is the topic of our work.

\section{Limitations}

Voronoi-WTA uses the WTA training scheme to estimate the inherent uncertainty in data. However, this approach has limitations and may achieve suboptimal performance. Recent studies have highlighted the sensitivity of WTA initialization in certain scenarios \cite{makansi2019overcoming, narayanan2021divide}. Exploring theoretically grounded solutions to address these issues, such as in \citet{arthur2007k}, could be a promising direction for future research. Additionally, there is no current evidence that the model can assess its own prediction confidence, such as in detecting out-of-distribution samples. Enhancing WTA learners with model uncertainty quantification could expand the abilities of WTA learners.

\section{Conclusion}

In this paper, we introduced \textit{Voronoi-WTA}, a novel conditional density estimator. Voronoi-WTA is a probabilistic extension of traditional WTA learners, leveraging the advantageous geometric properties of the WTA training scheme. Notably, Voronoi-WTA demonstrates greater resilience to the choice of scaling factor $h$ compared to the more straightforward Kernel-WTA. We support our claims with mathematical derivations, discussing the asymptotic performance as the number of hypotheses increases. Both theoretical analysis and experimental comparisons against several baselines highlight the strengths of our estimator. The application of our estimator to more realistic datasets opens up broad possibilities for future work.

\section*{Acknowledgements}

This work was funded by the French Association for Technological Research (ANRT CIFRE contract 2022-1854) and Hi! PARIS through their PhD in AI funding programs. We are grateful to the reviewers for their insightful comments.

\section*{Impact Statement}

This paper presents work whose goal is to advance the field of Machine Learning. There are many potential societal consequences of our work, none which we feel must be specifically highlighted here.

\label{submission}

\bibliography{main_paper}
\bibliographystyle{icml2024}

%%%%%%%%%%%%%%%%%%%%%%%%%%%%%%%%%%%%%%%%%%%%%%%%%%%%%%%%%%%%%%%%%%%%%%%%%%%%%%%
%%%%%%%%%%%%%%%%%%%%%%%%%%%%%%%%%%%%%%%%%%%%%%%%%%%%%%%%%%%%%%%%%%%%%%%%%%%%%%%
% APPENDIX
%%%%%%%%%%%%%%%%%%%%%%%%%%%%%%%%%%%%%%%%%%%%%%%%%%%%%%%%%%%%%%%%%%%%%%%%%%%%%%%
%%%%%%%%%%%%%%%%%%%%%%%%%%%%%%%%%%%%%%%%%%%%%%%%%%%%%%%%%%%%%%%%%%%%%%%%%%%%%%%
\newpage
\appendix
\appendix
\onecolumn

\section*{Organization of the Appendix}

The Appendix is organized as follows. Appendix \ref{secapp:notations} outlines the notations employed. Appendix \ref{secapp:theory} presents the theoretical results of the paper, detailing the background in Appendix \ref{sec:thsetup}, demonstrating how our estimator performs distribution estimation in Appendix \ref{secapp:distribution_estimation}, and discussing its geometric properties in Appendix \ref{secapp:geometrical_properties}. Appendix \ref{sec:experimental_details} covers the experimental details and design choices, including the synthetic data experiments in Appendix \ref{secapp:synthetic_data_experiments}, the UCI regression benchmark in Appendix \ref{app:uci-datasets} and the audio experiments in Appendix \ref{app:audio}.

\section{Notations}
\label{secapp:notations}

Let $\cX$ be the set of possible inputs, and by $\cY$ the set of possible targets. We assume that $\cX$ and $\cY$ are finite-dimensional real vector spaces, and we note $d$ as the dimension of $\cY$.
Given $\nhyp$ hypotheses, let $f_{\theta}=(f_{\theta}^1,\dots,f_{\theta}^{\nhyp})\in\cF(\cX,\cY^\nhyp)$ and $\gamma_{\theta}=(\gamma_{\theta}^{1},\dots,\gamma_{\theta}^{\nhyp})\in\cF(\cX,\Delta_\nhyp)$ be the predictions and scoring models, where $\Delta_\nhyp=\{p\in[0,1]^\nhyp \;\; \sum_{k=1}^\nhyp p_k = 1\}$ is the simplex on $\bR^\nhyp$.
These models are described by parameters $\theta$.
When not necessary, we omit this dependency by writing $z_k=\fk$ and $\gamma_k=\gamma_{\theta}^k$. For a given input $x\in\cX$ the set of predictions $\{f_{\theta}^{k}(x)\}_{k\in\bn}$ induces a Voronoi tessellation of the target space $\cY$. We note $\mathcal{Y}^k_{\theta}(x)$, or alternatively $\AYk(x)$, the Voronoi cell $k$, and by $\zkx=\fkx$ its generator. Recall that 
$$\AYk(x)=\left\{y\in\cY \;|\; \ell(y,\Azk(x)) < \ell(y,z_l(x)), \forall l \neq k \right\}\,,$$

where $\ell: \mathcal{Y} \times \mathcal{Y} \rightarrow \mathbb{R}_{+}$ is the underlying loss used, for instance the $\mathrm{L}^2$ loss  $\ell(\hat{\by}, \by) = \|\hat{\by}-\by\|^2$, denoting by $\|\cdot\|$ the Euclidean norm.

We will drop the dependency on $x$ when the context is clear, thus referring to $\AYk$ and $\Azk$. Conversely, when we study asymptotic properties that depend on the number of hypotheses $\nhyp$, we will emphasize this dependency by writing $\AYkn$ and $\Azkn$. Additionally, when $\cY$ is a $d$-dimensional cube, it can be partitioned into a regular grid. We note $\cG_k$ the cubes of this grid, and by $g_k$ the center of each cube. Note that this is a special case of Voronoi tessellation.

We note the set sum by $E+F=\{e+f, e\in E, f \in F \}$, and the border of a set $E$ by $\partial E=\bar{E} \backslash \mathring{E}$ where $\bar{E}, \mathring{E}$ represents the closure and interior of $E$ respectively.
$B(0, r)$ is the ball of radius $r$ with center $0$ associated to $\|\cdot\|$.
We note $|E|$ the cardinal of a set $E$, and by $\Delta(E)=\sup_{x, y \in E}\|x-y\|$ its diameter.

We note $\bP$ a data distribution over $\cX\times\cY$, $\lambda$ the Lebesgue measure, $\delta_y$ the Dirac measure centered on $y$, $\mathds{1}$ the indicator function, $\mathcal{U}$ the uniform distribution, and $\cN(a,b)$ the normal distribution with mean $a$ and variance $b$.
We will always assume that $\bP$ admits a probability density function $\rho(x,y)$. We denote $\bP_x$ and $\rho_x$ as the distribution and density, respectively, conditional on $x$. If $p$ and $q$ denote two densities over a domain $\mathcal{D}$, we define the Kullback–Leibler divergence as
\begin{align*}
  \mathrm{KL}(p || q) &= \int_{\mathcal{D}} \log\left(\frac{p(x)}{q(x)}\right)p(x)\dx\,.
\end{align*}
When the two distributions are discrete, for instance with a support of size $\nhyp$, we will write
$$\mathrm{KL}_{k\in\bn}(p_k \;||\; q_k) = \sum_{k=1}^\nhyp \log\left(\frac{p_k}{q_k}\right)p_k\,.$$
For scalars $a,b \in (0,1)$, we define the binary cross entropy (BCE) as
$$\mathrm{BCE}(a,b) = - a\log(b) - (1-a)\log(1-b)\,,$$
adopting the convention that $0 \log 0 = 0$.

In the following, we define training objectives, which are functions of model parameters $\theta$, using the notation $\mathcal{L} \triangleq \mathcal{L}(\theta)$. For a specific model $\mathrm{M}$, the training objective is denoted by $\mathcal{L}_{\mathrm{M}}(\theta)$. The single-sample version of this objective, which we will denote as $\mathcal{L}^{\mathrm{M}}(\theta)$ for brevity, is expressed for individual data points $(x,y)$ as $\mathcal{L}^{\mathrm{M}}(\theta,x,y)$. This single-sample loss contributes to the overall objective $\mathcal{L}_{\mathrm{M}}(\theta)$ through integration over the data distribution $\rho(x,y)$, with $\mathcal{L}_{\mathrm{M}}(\theta) = \int_{\mathcal{X} \times \mathcal{Y}} \mathcal{L}^{\mathrm{M}}(\theta,x,y) \rho(x,y) \mathrm{d}x \mathrm{d}y$.

\section{Theoretical results}
\label{secapp:theory}
The estimator Voronoi-WTA (V-WTA) introduced in this paper, has two main advantages:
1) it accurately estimates the data distribution,
and 2) the centroids, aligning with the optimal hypotheses according to Proposition \ref{th:du_et_al}, preserve the geometry of the data distribution.
As a result, this method effectively combines the strengths of Mixture density networks and Winner-takes-all models.
In this section, we will study these two claims along two main axes: convergence in distribution, and asymptotic quantization risk.

The section is organized as follows. We will first introduce the necessary definitions as well as our working hypotheses, then focus on distribution estimation, and finally study the geometrical properties of our proposed algorithm. 

\subsection{Theoretical setup}
\label{sec:thsetup}
\subsubsection{Background}

We are concerned with various estimators of the conditional distribution $\bP_x$, and study their convergence. We will make use of the Portmanteau lemma \cite{van2000asymptotic} and define weak convergence as follows.
\begin{definition}[Weak convergence]
    We say that a sequence of measures $(\mathbb{P}_\nhyp)_{\nhyp\in\mathbb{N}}$ converges weakly towards a measure $\mathbb{P}$, and we write $\bP_\nhyp\underset{\nhyp \to +\infty}{\rightharpoonup}\bP$, if $\mathbb{P}_\nhyp(E)\underset{\nhyp \rightarrow+\infty}{\longrightarrow}\mathbb{P}(E)$ for all measurable $E$ satisfying $\mathbb{P}(\partial E)=0$. 
\end{definition} 
We also study the convergence of sequences with finite support. In this context, we will often discuss uniform convergence, which we define below.
\begin{definition}[Uniform convergence]
    Let $(u_{\nhyp,k})_{(\nhyp,k)\in\mathbb{N}^2}$ denote a sequence such that $(u_{\nhyp,k})_{k\in\mathbb{N}}$ has finite support for each $\nhyp$. We will say that $u$ converges uniformly towards $(v_k)_{k\in\mathbb{N}}$ if $\underset{k\in\mathbb{N}}{\max}\|u_{\nhyp,k}-v_k\|\underset{\nhyp \rightarrow+\infty}{\longrightarrow} 0$. In particular, $u$ vanishes uniformly if $\underset{k\in\mathbb{N}}{\max} \|u_{\nhyp,k}\|\underset{\nhyp \rightarrow+\infty}{\longrightarrow} 0$.
\end{definition} 

In what follows, we will extensively use Zador's theorem \cite{zador1982asymptotic}, a powerful result on the asymptotic distribution of the centroids resulting from optimal quantization, which we recall below (see \citet{graf2008distortion}, Equation 2.3, or \citet{iacobelli2016asymptotic}, Theorem 1.3, for a more general version). This theorem will allow us to derive asymptotic properties of Winner-takes-all models.

\begin{theorem}[Zador theorem]
    Let $\bP = \rho \: \dy$ be a Lebesgue-dominated probability measure on a compact subset $\cY$ of $\bR^d$. Define the optimal quantization risk 
    $$\distortion_\nhyp(\bP)=\inf_{\cZ\subset\cY:|\cZ|\leq \nhyp} \int_\cY \min_{z\in\cZ} \|y-z\|^2 \rho(y)\dy\,,$$
    and the asymptotic  risk for the uniform distribution $J_d = \inf_\nhyp \nhyp^{2 / d} \distortion_\nhyp(\mathcal{U}([0,1]^d).$
    Then 
    $$\lim _{\nhyp \rightarrow +\infty} \nhyp^{2 / d} \distortion_\nhyp(\bP)=J_d \; \left(\int_\cY \rho^{d /(d+2)} \dy\right)^{(d+2) / d}\,.$$
    In addition, if $\mathcal{Z}$ minimizes the risk $\distortion_\nhyp(\mathbb{P})$, then 
    $$\dfrac{1}{\nhyp}\sum_{z\in\mathcal{Z}}\delta_z \underset{\nhyp \to \infty}{\rightharpoonup} \frac{\rho^{d / (d+2)}}{\int_\cY \rho^{d / (d+2)}(x) d x} \dy\,.$$
\end{theorem}

The constant $J_d$ can be computed for simple cases ($J_1=\frac{1}{12}$ and $J_2=\frac{5}{18\sqrt{3}}$ \cite{newman1982hexagon}) and can be approximated for large $d$ by $J_d \sim \dfrac{d}{2\pi e}$ \cite{pages2003optimal, graf2007foundations}.

\subsubsection{Estimators}

Using these notations we can define several estimators of the conditional distribution $\bP_x$, for each $x \in \cX$:
\begin{align}
\mathbb{P}_x^{\;\text{MDN}} &:E \mapsto\sum_{k=1}^\nhyp \pi_k(x) \; \mathcal{N}\left(E ; \mu_k(x), \Sigma_k(x)\right) \\
\mathbb{P}_x^{\;\text{H}} &:E \mapsto\sum_{k=1}^\nhyp \gamma_k(x) \frac{\lambda\left(E \cap \cG_k\right)}{\lambda\left( \cG_k\right)} \\
\pdwta &:E \mapsto\sum_{k=1}^\nhyp \gamma_k(x) \delta_{\Azk}(E) \\
\puwta &: E \mapsto \sum_{k=1}^\nhyp \gamma_k(x) \frac{\lambda\left(E \cap \AYk\right)}{\lambda\left( \AYk\right)} \\
\pkwta &: E \mapsto \sum_{k=1}^\nhyp \gamma_k(x) K_h(\Azk(x),E)\\
\ptwta&: E \mapsto \sum_{k=1}^\nhyp \gamma_k(x) \frac{K_h\left(\Azk(x), E \cap \AYk \right)}{K_h\left(\Azk(x), \AYk \right)},
\end{align}
where $h \in \mathbb{R}_{+}^{*}$ is the scaling factor of $K_h$ and $K_h(\Azk(x),E) \triangleq \int_{E} K_h(\Azk(x),y) \dy$.

Note that we obtain the variants of the original Dirac estimator $\pdwta$ \cite{letzelter2023resilient} by changing the Dirac kernel to a uniform kernel ($\puwta$), the kernel $K_h$ ($\pkwta$), or its truncated version ($\ptwta$). 

When there is no ambiguity, we will refer to these estimators by $\hat{\bP}_x$, and their density by $\hat{\rho}_x$ (when it exists).

\subsubsection{Training objectives}

We recall that Winner-takes-all models are trained with two objectives: a quantization objective optimizing the position of the hypotheses $\Azk(x)$ and a scoring objective enforcing that $\gamma_k(x)$ accurately estimates the probability $\bP(\AYk(x))$ of each Voronoi cell of the tessellation. 
induced by the hypotheses. More specifically,

\begin{align}
    \cL_{\text{centroid}}(\cZ) &= \int_\cX\int_\cY \min_{z\in\cZ_x} \ell(z,y)
    \rho(x,y)\dx\dy\,, \\
    \cL_{\text{scoring}}(\gamma) &= \int_\cX\int_\cY \sum_{k=1}^\nhyp \mathrm{BCE}\left[\bI\left[y\in\AYk(x)\right],\gamma_k(x) \right]\rho(x,y)\dx\dy\,,
\end{align}
where $\cZ: x \mapsto \cZ_x = \{z_k(x)\}_{k\in\bn} \subset \mathcal{Y}$ and $\gamma: x \mapsto (\gamma_k(x))_{k \in \bn} \in \Delta_{\nhyp}$.

\subsubsection{Assumptions}

Throughout our analysis, we will often use the following assumptions. 
\begin{assumption}[Boundedness]\label{hyp:compact}
    The set of possible outputs $\cY$ is compact. 
\end{assumption}
\begin{assumption}[Positivity]\label{hyp:positive}
    The data probability density function (PDF) $\rho$ satisfies $\inf_{(x,y)\in\cX\bigtimes\cY}\;[\rho(x,y)]>0$. 
\end{assumption}
\begin{assumption}[Lipschitz]\label{hyp:lipschitz}
    The conditional data PDF $\rho_x$ is $L$-lipschitz for each $x \in \cX$. 
\end{assumption}
\begin{assumption}[Optimality]\label{hyp:minimum}
    The Winner-Takes-All algorithm has converged toward a global minimum of its centroid objective
    \begin{align}\label{eqn:centroid_objective}
        \min_{\cZ}\int_\cX\int_\cY \min_{z\in\cZ_x} \ell(z,y) \rho(x,y)\dx\dy\,,    
    \end{align}
    and its scoring objective (noting $x\mapsto(\AYk(x))_{k\in\bn}$ the resulting optimal voronoi tesselation map)
    \begin{align}\label{eqn:scoring_objective}
        \min_{\gamma}  \int_\cX \int_\cY \sum_{k=1}^\nhyp \mathrm{BCE}\left[\bI\left[y\in\AYk(x)\right] ,\gamma_k(x)\right]\rho(x,y)\dx\dy\,.  
    \end{align}
\end{assumption}

An empirical discussion of these assumptions is given in Appendix \ref{sec:discussion}.

\subsection{Distribution estimation}
\label{secapp:distribution_estimation}
\subsubsection{Unbiased estimators}

The first interesting property of WTA is that its scoring model is an unbiased estimator of the Voronoi cell's probability mass 
\begin{proposition}\label{th:score}
    Under Assumption \ref{hyp:minimum}, we have 
    $$\forall x\in\cX, \quad\forall k\in\bn,\quad \gamma_k(x)=\bP_x(\AYk(x))\,.$$
\end{proposition}

This observation is key to establishing other interesting properties of WTA. Note that it is independent of the kernel choice, so it applies to all variants of WTA. 
\begin{proof}
The scoring objective will be minimal when the integrand of Equation \ref{eqn:scoring_objective} is minimal for each $x\in\cX$. Looking only at the integrand, we can write the following.

\begin{align*}
    \int_\cY \sum_{k=1}^\nhyp \mathrm{BCE}\left[\bI[y\in\AYk(x)] , \gamma_k(x)\right]\rho_x(y)\dy 
    &= - \sum_{k=1}^\nhyp \int_{\AYk(x)} \log(\gamma_k(x))\rho_x(y)\dy + \int_{\cY\setminus\AYk(x)}\log(1-\gamma_k(x)) \rho_x(y)\dy \\
    &= - \sum_{k=1}^\nhyp \log(\gamma_k(x)) \bP(\AYk(x))+ \log(1-\gamma_k(x))(1-\bP(\AYk(x)))
\end{align*} 
For each $k$ in the sum, we recognize a binary cross-entropy between $\gamma_k(x)$ and $\bP_x(\AYk(x))$, which is minimal when the two terms are equal.
\end{proof}

Therefore all truncated estimators are themselves unbiased estimators of the Voronoi cell's probability mass. We refer to this as the cell-scoring property, defined in \eqref{eq:scoring-prop}.

\begin{proposition}\label{th:unbiased}
    Under Assumption \ref{hyp:minimum}, all estimators except $\mathbb{P}_x^{\;\mathrm{MDN}}$ and $\pkwta$ satisfy
    $$\forall x\in\cX, \quad \forall k\in\bn, \quad \hP(\AYk(x))=\bP(\AYk(x))\,.$$
\end{proposition}

\begin{proof}
    Corollary of Proposition \ref{th:score}.
\end{proof}

\subsubsection{Interpretation of the negative log-likelihood}

The negative log-likelihood (NLL) is a useful quantity for measuring the accuracy of a density estimator. For instance, Mixture Density Networks minimize the NLL during training. Score-based WTA models are not trained to directly minimize NLL. 
The following result states that, in the case of uniform-kernel estimators, the scoring objective and the NLL are minimized when high-density zones of the target space $\cY$ are assigned to smaller Voronoi cells in volume.

\begin{proposition}
Under Assumption \ref{hyp:minimum}, the estimator $\mathbb{P}_x^{\;\mathrm{U-WTA}}$ conditional negative log-likelihood satisfies for each $x \in \cX$:
\begin{equation}
\mathrm{NLL}(\mathbb{P}_x^{\;\mathrm{U-WTA}}, \mathbb{P}_x)=- \sum_{k=1}^\nhyp \log\frac{\bP_x(\AYk(x))}{\lambda\left(\AYk(x)\right)} \; \bP_x(\AYk(x)) \triangleq-\mathrm{KL}_{k\in\bn}\left[\bP_x(\AYk(x))\mid \mid \dfrac{\lambda(\AYk(x))}{\mathrm{vol}(\mathcal{Y})} \right] + \log \mathrm{vol}(\mathcal{Y})\,.
\label{eq:nll_kldiv}
\end{equation}

From \eqref{eq:nll_kldiv}, we see that minimizing the NLL with constant volume $\mathrm{vol}(\mathcal{Y})$ requires strategic placement of hypotheses. Specifically, the probabilities $\bP_x(\AYk(x))$ should be high in regions where the relative volume $\dfrac{\lambda(\AYk(x))}{\mathrm{vol}(\mathcal{Y})}$ is low.
\end{proposition}

\begin{proof}
We note $\hat{\rho}_x$ the density of the estimator $\mathbb{P}_x^{\;\text{U-WTA}}$. By definition \eqref{eq:nll}, $\mathrm{NLL}(\hat{\rho}_x, \rho_x) \triangleq - \int_\cY \log(\hat{\rho}_x(y)) \rho_x(y)\dx\dy$. 

We can write
\begin{align}
    \int_\cY \log(\hat{\rho}_x(y)) \rho_x(y)\dy &\triangleq  \sum_{k=1}^\nhyp \int_{\AYk(x)} \log\frac{\gamma_k(x)}{\lambda\left(\AYk(x)\right)}\rho(x,y)\dy \label{eq:nll_formula}\\
    &=  \sum_{k=1}^\nhyp \log\frac{\bP_x(\AYk(x))}{\lambda\left(\AYk(x)\right)} \; \bP_x(\AYk(x)) \quad \text{(using proposition \ref{th:score})}\notag \\
    &=  \mathrm{KL}_{k\in\bn}\left[\bP_x(\AYk(x))\mid \mid \dfrac{\lambda(\AYk(x))}{\mathrm{vol}(\mathcal{Y})} \right] - \log (\mathrm{vol}(\cY) ) \left( \sum_{k=1}^{K} \bP_x(\AYk(x)) \right)\,.\notag
\end{align}
\end{proof}

\subsubsection{Convergence of the estimators}

We now turn to the main result of this section: the estimators $\pdwta$, $\puwta$, and $\ptwta$ converge in distribution towards the true data distribution $\bP_x$. 

This result is similar to Theorem 4.1 in \citet{polianskii2022voronoi} about the convergence of the Compactified Voronoi Density Estimator (CVDE). 
Our proof mirrors the one they propose in this article, which is itself a slight reformulation of the Theorem 5.1 of \citet{devroye2017measure}. However, our setting differs from these two articles: the authors consider random \textit{i.i.d.} generators $\Azk\sim\bP_x$ (similarly to Kernel Density Estimation \cite{rosenblatt1956remarks, gramacki2018nonparametric}). However, this assumption is not satisfied in the context of WTA, which makes their result less relevant to our purpose. To correct this mismatch, we investigate the more realistic assumption that $\Azk$ minimizes the centroid objective (Assumption \ref{hyp:minimum}).
This makes our analysis more relevant in the context of Winner-Takes-All-based models.
Additionally, we study the more general problem of conditional density estimation.

The proof of CVDE convergence relies on the intuitive observation that Voronoi cells' radius vanishes as the number of centroids $\nhyp$ increases. Using the Zador theorem, we first show that this phenomenon still holds when the centroids are selected according to optimal quantization.

\begin{proposition}\label{th:radius}
    Under Assumption \ref{hyp:compact} (boundedness), \ref{hyp:positive} (positive density) and \ref{hyp:minimum} (optimal centroids), the Voronoi cell diameter $\Delta(\AYkn)$ vanishes uniformly. 
\end{proposition}

\begin{figure}
    \centering
    \includegraphics[width=0.4\textwidth]{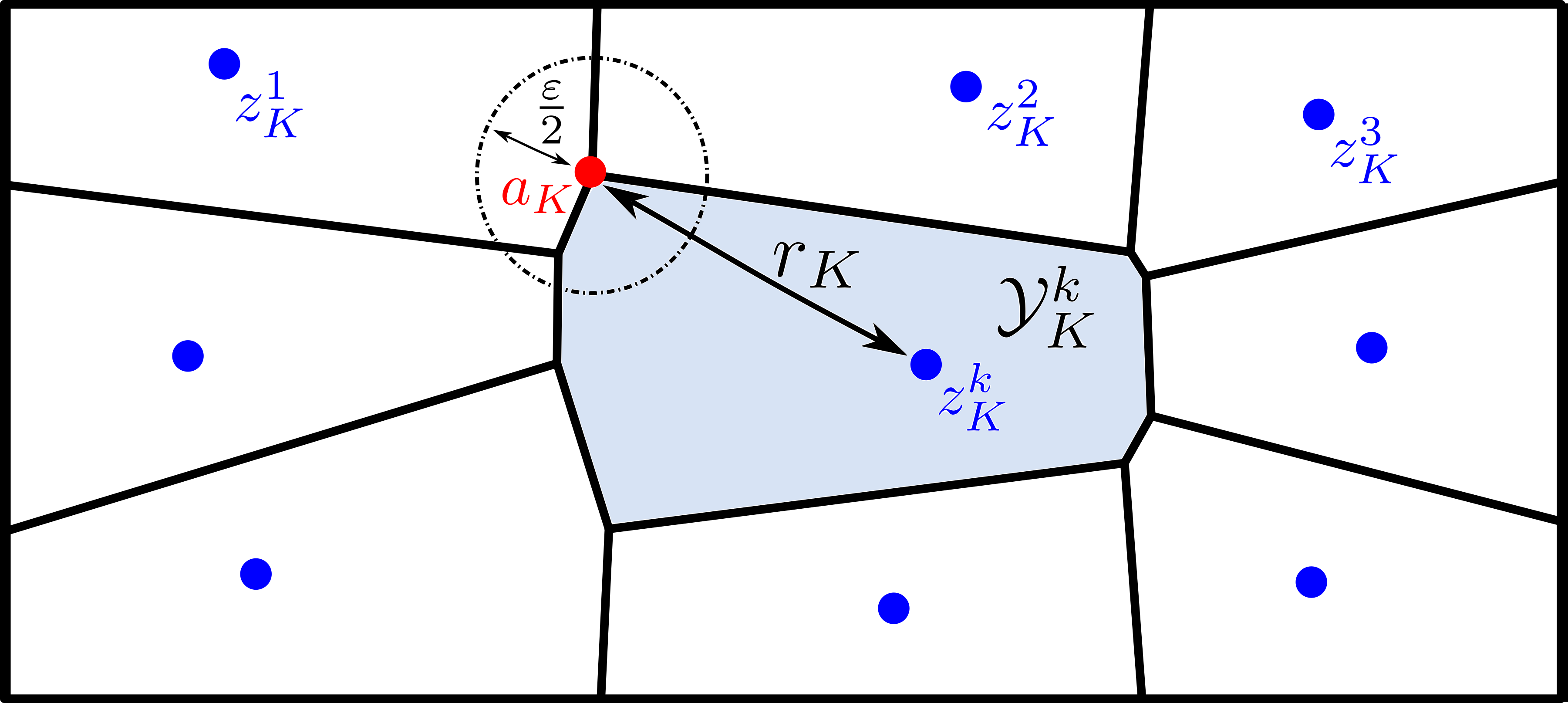}
\hspace{0.5cm}\includegraphics[width=0.2\textwidth]{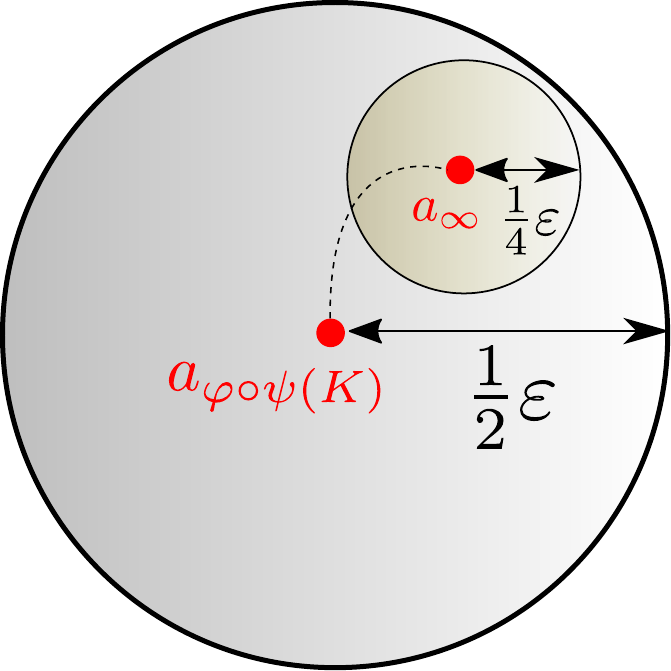}
    \caption{\textbf{Illustration of the proof of Proposition \ref{th:radius}.} On the left, we show that $a_\nhyp$ is exactly $r_\nhyp$ apart from its closest centroid. On the right, we illustrate the sequence $a_{\varphi\circ\psi(\nhyp)}$, from which we define $B_\infty$.}
    \label{fig:radius}
\end{figure}

\begin{proof}
    Suppose that the diameter $\Delta({\AYkn})$ does not vanish. Infinitely often, there are some cells $\AYkn$ that have a diameter greater than some $\varepsilon>0$. Inside these cells, there are points $y$ that are more than $\frac{\varepsilon}{2}$ apart from the closest centroid. In other words, there are balls of radius $\frac{\varepsilon}{2}$ which contain no centroids (see Figure \ref{fig:radius}). This is in contrast with the second statement of Zador's theorem, which stipulates that the centroids $\Azkn$ become dense in $\cY$ as $\nhyp$ increases, hence a contradiction. We make this argument rigorous below.

    We note the cell radius $\rkn=\max_{y\in \AYkn} \|y-\Azkn\|$, and the maximal cell radius $r_\nhyp=\max_k \rkn$. It is more convenient to work with the cell radius than with their diameter. Note that $\Delta(\AYkn)\leq 2\rkn$ (triangular inequality), so that it is enough to show $r_\nhyp \underset{\nhyp \rightarrow \infty}{\longrightarrow} 0$.

    Let's assume the opposite, and let $\varphi$ be a subsequence satisfying $\forall \nhyp\in\mathbb{N}, \; r _{\varphi(\nhyp)} \geqslant \varepsilon$, for some $\varepsilon>0$. We will build a ball $B_\infty$ which contains no centroid infinitely often (see Figure \ref{fig:radius}). 
    
    Let $a_\nhyp \in \underset{y \in \mathcal{Y}}{\arg \max }\left\{\underset{k\leq \nhyp}{\mathrm{min}}\left\|y-z^k_{\nhyp}\right\|\right\}$ be a point in the set of farthest points from their respective centroid. We can see that $B\left(a_{\varphi(\nhyp)}, \frac{\varepsilon}{2}\right)$ does not intersect any centroid at step $\varphi(\nhyp)$ (see Figure \ref{fig:radius}). Indeed, if $a_{\varphi(\nhyp)}\in \mathcal{Y}^k_{\varphi(K)}$ for some index $k$, then $z_{\varphi(\nhyp)}^k$ is its closest centroid and for all other centroid $z_{\varphi(\nhyp)}^l$, we have
    \begin{equation}
        \|a_{\varphi(\nhyp)}-z_{\varphi(\nhyp)}^l\|\geq\|a_{\varphi(\nhyp)}-z_{\varphi(\nhyp)}^k\|=r_{\varphi(\nhyp)}\geq\varepsilon>\frac{\varepsilon}{2}\,. \notag
    \end{equation}

    The sequence $(a_{\varphi(\nhyp)})_{\nhyp\in\bN}$ is bounded (Assumption \ref{hyp:compact}), so by Bolzano–Weierstrass theorem, there is a subsequence $\psi$ and a limiting point $a_\infty$ such that $a_{\varphi\circ\psi(\nhyp)} \underset{\nhyp \rightarrow \infty}{\longrightarrow} a_\infty$. If $\nhyp$ is large enough, $\|a_{\varphi\circ\psi(\nhyp)}-a_\infty\|\leq\frac{\varepsilon}{4}$ and consequently 
    \begin{equation}
        B_\infty\triangleq B\left(a_\infty, \frac{\varepsilon}{4}\right)\subset B\left(a_{\varphi\circ\psi(\nhyp)}, \frac{\varepsilon}{2}\right)\,.\notag
    \end{equation} 
    In particular, $B_\infty$ does not intersect any centroid at each step $\varphi\circ\psi(\nhyp)$, which is exactly what we wanted to achieve.
    
    The proportion of centroids contained in a set is given by the measure $\mathbb{P}_\nhyp=\frac{1}{\nhyp} \sum_{k=1}^\nhyp \delta_{\Azkn}$. The fact that $B_\infty$ does not intersect any centroid can be rewritten $$\forall \nhyp\in \varphi\circ\psi(\bN), \quad \mathbb{P}_{\nhyp}  \left(B_\infty\right)=\frac{1}{\nhyp}\left|\left\{ k \;|\; \Azkn \in B_\infty\right\}\right|=0\,.$$
    
    Assuming optimal centroid placement, we know from Zador's theorem that 
    $$\mathbb{P}_\nhyp \underset{\nhyp \rightarrow \infty}{\rightharpoonup} \frac{\rho^{d / (d+2)}}{\int_\cY \rho^{d / (d+2)}(x) \dx} \mathrm{~d} y \triangleq \rho_{\infty} \mathrm{~d} y \triangleq \mathbb{P}_{\infty}\,.$$
 
    It is clear from our hypotheses that $\inf_{y\in\cY}\rho_{\infty}(y) > 0$. We conclude with the two following contradicting observations.
    $$\forall \nhyp\in\bN \quad \mathbb{P}_{\varphi(\nhyp)}(B_\infty)=0 \quad\Rightarrow\quad \bP_\infty (B_\infty) = 0 \quad \text{(weak convergence)}$$
    $$\mathbb{P}_{\infty}(B_\infty) \geq \inf_{y\in\cY}\rho_{\infty}(y) \lambda(B_\infty)>0\,. \quad \text{(positivity)}$$
\end{proof}

We can now prove the convergence of WTA-based estimators.

\begin{proposition}\label{th:convergence}
    Under Assumption \ref{hyp:compact} (boundedness), \ref{hyp:positive} (positive density) and \ref{hyp:minimum} (optimal centroids), 
    $\pdwta$, $\puwta$, and $\ptwta$ converge weakly towards $\mathbb{P}$.
\end{proposition}

\begin{figure}
    \centering
        \includegraphics[width=0.47\textwidth]
        {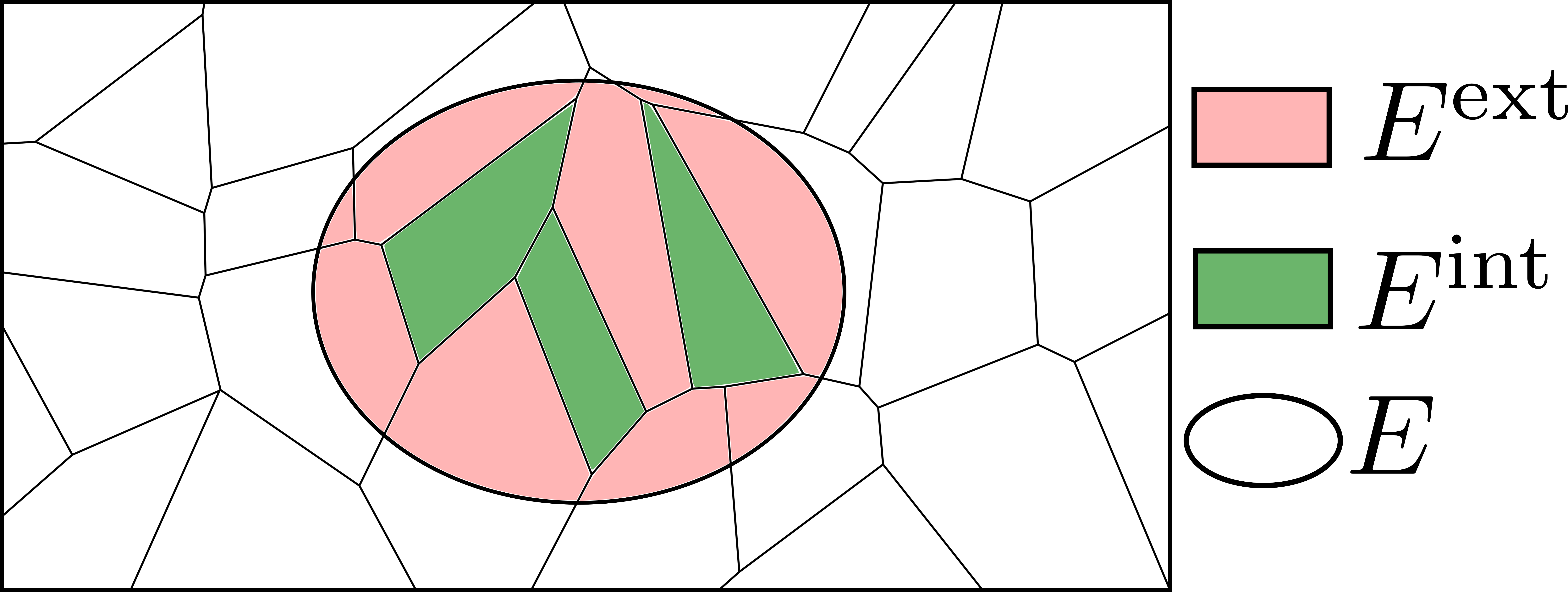}
    \caption{\textbf{Illustration of the partition of $E = E^{\text{int}} \cup E^{\text{int}}$ in the proof of Proposition \ref{th:convergence}.}}
    \label{fig:convergence}
\end{figure}

\begin{proof}
    Let $E$ denote any measurable such that  $\lambda(\partial E) = 0$. We want to show that $\mathbb{P}_\nhyp(E) \underset{\nhyp \rightarrow+\infty}{\longrightarrow} {\bP}(E)$. 
 
    If we fix $\nhyp$, the Voronoi tiling $\AYkn$ induces a partition of $E$. Accordingly, 
    we split $E$ as the disjoint union $E=E^{\text{int}}\cup E^{\text{ext}}$, where $E^{\text{int}}$ denotes the Voronoi cells included in $E$, and $E^{\text{int}}$ the Voronoi cells intersecting its border $\partial E$ (see Figure \ref{fig:convergence}). 
    Using the property that the Voronoi cell diameter $\Delta(\AYkn)$ vanishes uniformly as the number of hypotheses $\nhyp$ increases to infinity, we deduce that $E^{\text{ext}}$ is concentrated on the border $\partial E$ and is therefore negligible. The proof is concluded by observing that $\mathbb{P}_\nhyp$ and $\mathbb{P}$ coincide on $E^{\text{int}}$.
    We give the technical details below. 
    
    Let $$I^{\text{int}}_\nhyp=\left\{k \in \bn \;\mid\; \AYkn \cap E \neq \varnothing,\,\AYkn \backslash E=\varnothing\right\},$$ $$I^{\text{ext}}_\nhyp=\left\{k \in \bn \;\mid\; \AYkn \cap E \neq \varnothing,\, \AYkn \backslash E \neq \varnothing\right\},$$ $E^{\text{int}}_\nhyp = \cup_{k \in I^{\text{int}}} (\AYkn \cap E)$ and $E^{\text{ext}}_\nhyp = \cup_{k \in I^{\text{ext}}} (\AYkn \cap E)$.
    Clearly $E=E^{\text{int}}_\nhyp\cup E^{\text{ext}}_\nhyp$ and $E^{\text{int}}_\nhyp\cap E^{\text{ext}}_\nhyp=\varnothing$.

    Recall that the considered estimators have an interesting property: the estimated density of a Voronoi cell is equal to its true probability mass (Proposition \ref{th:unbiased}). Therefore, $\bP_\nhyp$ and $\bP$ coincide on $E^{\text{int}}_\nhyp$. More precisely,
    \begin{equation}
        \mathbb{P}_\nhyp(E^{\text{int}}_\nhyp)=\sum_{k \in I^{\text {int}}} \bP_\nhyp(\AYkn) =\sum_{k \in I^{\text {int}}} \mathbb{P}\left(\AYkn\right)=\mathbb{P}\left(E^{\text {int}}_\nhyp\right)\,.
        \label{eq:E_int}
    \end{equation}
    We then refer to the maximum cell diameter by $\varepsilon_\nhyp = \max_k \Delta(\AYkn)$, and $\varepsilon_\nhyp^+ = \sup_{k \geq \nhyp} \varepsilon_k$. We know from Proposition \ref{th:radius} that $\varepsilon_\nhyp \underset{\nhyp \rightarrow+\infty}{\longrightarrow} 0$.
    
    We now show that $E^{\text{ext}}_\nhyp \subset \partial E+B(0, \varepsilon_\nhyp^+).$

    Let $k\in I^{\text {ext}}_\nhyp$. We want to show that $\AYkn$ intersects $\partial E$ (see Figure \ref{fig:convergence}). By definition, $\AYkn$ is partially inside and outside $E$. Therefore, we can choose $x\in \AYkn \cap E$ and $y\in \AYkn \backslash E$. By convexity of the Voronoi cells, we can see that its border $\partial E$ will intersect the segment $[x,y]$ on a single point $y^*\in\AYkn$. Formally, let 
    $$t^*= \sup \{t \in[0,1] \; \mid \; (1-t) x+t y \in E\}\,, \quad\text{and}\quad y^*=\left(1-t^*\right) x+t^* y\,.$$ 
    We have $y^*\in\bar{E}$ because there is a sequence converging toward $y^*$ from inside $E$ by definition of $\sup$. Moreover, $y^*\notin E^{\mathrm{o}}$, because there would be $t>t^*$ satisfying the constraint, by definition of open sets. Finally $y^*\in \AYkn$ by convexity of $\AYkn$. Therefore $y^* \in \partial E \cap \AYkn$.
    
    By definition of the maximum diameter $\varepsilon_\nhyp^+$, $$y^*\in \AYkn \quad\Rightarrow\quad \AYkn \subset B(y^*, \varepsilon_\nhyp^+) = y^* + B(0, \varepsilon_\nhyp^+) \subset \partial E+B(0, \varepsilon_\nhyp^+)\,.$$

    Therefore, $E^{\text{ext}}_\nhyp \subset \cup_{k \in I^{\text{ext}}} \AYkn \subset \partial E+B(0, \varepsilon_\nhyp^+)$. We conclude by observing that $\AYkn$ are disjoint, and that $\mathbb{P}_\nhyp(\AYkn\cap E) \leq \mathbb{P}_\nhyp(\AYkn)=\mathbb{P}(\AYkn)$ for the considered estimators (Proposition \ref{th:unbiased}).
    \begin{equation}
        \mathbb{P}_\nhyp(E^{\text{ext}}_\nhyp)=\mathbb{P}_\nhyp(\cup_{k\in I^{\text{ext}}} (\AYkn \cap E))\leq \mathbb{P}(\cup_{k\in I^{\text{ext}}} (\AYkn))\leq\mathbb{P}(\partial E+B(0, \varepsilon_\nhyp^+))\underset{\nhyp \rightarrow+\infty}{\longrightarrow}\mathbb{P}(\partial E)=0\,. \label{eq:pn_ext}
    \end{equation}
    Likewise, 
    \begin{equation}
        \mathbb{P}(E^{\text{ext}}_\nhyp)\leq \mathbb{P}(\partial E+B(0, \varepsilon_\nhyp^+)) \underset{\nhyp \rightarrow+\infty}{\longrightarrow}\mathbb{P}(\partial E)=0\,. \label{eq:p_ext}
    \end{equation}
    In conclusion,
    $$|\mathbb{P}_\nhyp(E)-\mathbb{P}(E)| \leq \underbrace{|\mathbb{P}_\nhyp(E^{\text{int}}_\nhyp)-\mathbb{P}(E^{\text{int}}_\nhyp)|}_{=\;0 \; \text{by Eq. (\ref{eq:E_int})}}+\underbrace{|\mathbb{P}_\nhyp(E^{\text{ext}}_\nhyp)-\mathbb{P}(E^{\text{ext}}_\nhyp)|}_{\rightarrow \;0 \; \text{by (\ref{eq:pn_ext}) and Eq. (\ref{eq:p_ext})}}\underset{\nhyp \rightarrow+\infty}{\rightarrow}0\,.$$
\end{proof}

Note that Eq. (\ref{eq:E_int}) does not necessarily hold for $\pkwta$. Therefore, this proof cannot apply to this estimator. However, it applies to a large family of estimators $\hP$. Indeed, it is independent of the choice of kernel, as long as $\hP$ is an unbiased estimator of the Voronoi cell probability mass.

\subsection{Geometrical properties}
\label{secapp:geometrical_properties}

We will study the geometrical properties of WTA-based models through the lens of quantization risk. Our analysis will focus on the estimators $\pdwta$, $\ptwta$, and $\pkwta$, as these correspond to the cases presented in the main paper. Specifically, we will concentrate on the positions of the hypotheses, an aspect for which the same analysis applies to all three estimators. Indeed, both $\ptwta$ and $\pkwta$ retain the hypotheses positions following WTA training. For the rest of this section, we drop the dependency on $x$ without loss of generality, to lighten the notational burden. 

\subsubsection{Quantization risk}

We are interested in measuring the advantage of an adaptative grid of WTA. To do so, we look at the quantization risk of WTA and Histogram.

\begin{proposition}\label{th:risk_comparison}
    Under Assumption \ref{hyp:minimum}, we have for each $\nhyp \in\bN$

    $$ \sum_{k=1}^{\nhyp} \int_{\AYk} \|\Azk-y\|_2^2 \rho(y) \dy \leq \sum_{k=1}^{\nhyp} \int_{\mathcal{G}^k} \|g_k-y\|_2^2 \rho(y) \dy\,.$$
\end{proposition}

\begin{proof}
    The histogram is a particular case of Voronoi tesselation.
\end{proof}

\subsubsection{Asymptotic quantization risk}

We know the asymptotic quantization risk of WTA from the Zador theorem.

\begin{proposition}\label{th:wta_risk}
    Under Assumptions \ref{hyp:compact} (boundedness) and \ref{hyp:minimum} (optimality), the asymptotic quantization risk of the estimator $\mathbb{P}_x^{\;\mathrm{D-WTA}}$ has the following asymptotic evolution as $K \rightarrow \infty$
    $$\distortion_\nhyp^{\mathrm{D-WTA}}=J_d \; \left(\int_\cY \rho^{d /(d+2)} \dy\right)^{(d+2) / d}\frac{1}{\nhyp^{2/d}} + o\left(\frac{1}{\nhyp^{2/d}}\right)\,.$$ 
\end{proposition}

\begin{proof}
    It is a corollary of the Zador theorem, whose conditions are met given our hypotheses.
\end{proof}

It is also possible to compute the risk for the grid estimator $\ph$.

\begin{proposition}\label{th:histogram_risk}
    Under Assumptions \ref{hyp:compact} (boundedness), \ref{hyp:positive} (positivity), \ref{hyp:lipschitz} (lipschitz), \ref{hyp:minimum} (optimality), and assuming moreover that $\cY=[0,c]^d$ , the quantization risk of the estimator $\ph$ has the following asymptotic evolution
    $$\distortion_\nhyp^{\mathrm{H}} =\frac{d}{12}\frac{c^2}{\nhyp^{2/d}} + \mathcal{O}\left(\frac{1}{\nhyp^{3/d}}\right)\,.$$
\end{proposition}

\begin{proof}
Consider a $d$-dimensional grid of $\nhyp = M^d$ points. Admitting for now that the risk over a $d$-dimensional cube of size $\frac{c}{M}$ is equal to $\frac{d}{12}\left(\frac{c}{M}\right)^{d+2}$, we can rewrite the risk as follows. 
$$\distortion_{M^d}^{\mathrm{H}}=\sum_{k=1}^{M^d} \int_{\AYk} \rho(y) \|y-\Azkn\|^2 \dy \approx\sum_{k=1}^{M^d} \rho(\Azkn)\int_{\AYk} \|y-\Azkn\|^2 \dy = \frac{d}{12}\frac{c^2}{M^2}\left(\frac{c^d}{M^d} \sum_{k=1}^{M^d} \rho(\Azk)\right) \approx 
\frac{d}{12}\frac{c^2}{M^2}\,.$$
The first approximation says that $\rho$ is essentially constant over a cube cell $V_k$ if $\nhyp$ is large enough. The second approximation is a Monte Carlo integration. If we substitute $M$ by $\nhyp$, we obtain the announced result.

We now justify these two approximations formally.
We first compute the risk over a $d$-dimensional cube of side $a > 0$ centered in $0$ (considering a uniform distribution for $\rho$).
$$ \int_{[-\frac{a}{2}, \frac{a}{2}]^d} \left(x_1^2+\dots+x_d^2\right) \mathrm{d} x_1 \dots \mathrm{d} x_d = d \int x_1^2 \mathrm{d} x_1 \dots \mathrm{d} x_d = d a^{d-1} \int_{-a / 2}^{a / 2} x_1^2 \mathrm{d} x_1 = \frac{d}{12} a^{d+2}\,.$$

Now we compute an upper bound of the $\alpha$-distortion (defined below) over the same cube. This upper bound relies on enclosing the cube in the smallest ball containing it (which has for diameter the largest diagonal of the cube  $\|(a,\ldots,a)-(0,\ldots,0)\|=\sqrt{d}a$).
\begin{align*}
    \int_{[-\frac{a}{2}, \frac{a}{2}]^d}\left(\sqrt{x_1^2+\cdots+x_d^2}\right)^\alpha \mathrm{d} x_1\cdots \mathrm{d} x_d &\leqslant \int_{B\left(0, \sqrt{d} a/2\right)} \left(\sqrt{x_1^2+\cdots+x_d^2}\right)^\alpha \mathrm{d} x_1\cdots \mathrm{d} x_d \\ 
    &= \int_0^{\sqrt{d} \frac{a}{2}} r^\alpha \frac{2\pi^{d/2}}{\Gamma\left(\frac{d}{2}\right)} r^{d-1} \mathrm{d} r \\
    &=\frac{2 \pi^{d/2}}{\Gamma\left(\frac{d}{2}\right)}  \int_0^{\sqrt{d} \frac{a}{2}} r^{d-1+\alpha} \mathrm{d} r \\
    &=\mathcal{O}\left(a^{d+\alpha}\right),
\end{align*}
Where $\Gamma: x \in \mathbb{R}_{+}^{*} \mapsto \int_{0}^{\infty} t^{x-1} e^{-t} \mathrm{d}t$ is the gamma function. Equipped with this result, we can prove the first approximation using the assumption that $\rho$ is $L$-lipschitz. 
\begin{align*}
    \left|\distortion_\nhyp^{\mathrm{H}} - \sum_{k=1}^\nhyp \rho(z^k_{\nhyp})\int_{\mathcal{Y}^k_{\nhyp}} \|y-z^k_{\nhyp}\|^{2} \dy \right| &\leqslant \sum_{k=1}^\nhyp \int_{\mathcal{Y}^k_{\nhyp}}|\rho(y)-\rho(z^k_{\nhyp})|\;\|y-z^k_{\nhyp} \|^2 \dy \\
    &\leqslant L \sum_{k=1}^\nhyp \int_{\mathcal{Y}^k_{\nhyp}}\|y-z^k_{\nhyp}\|^3 \dy \\
    &= L \sum_{k=1}^{\nhyp} \mathcal{O}\left(\frac{1}{M^{d+3}}\right)\\
    &= \mathcal{O}\left(\frac{1}{M^3}\right)\,.
\end{align*}
The second approximation is similar.
\begin{align*}
    \left|\frac{1}{M^d}\sum_{k=1}^\nhyp \rho(z^k_{K}) - \int_{\cY} \rho(y) \dy \right| 
    &= \left|\sum_{k=1}^\nhyp \rho(z^k_{K})\int_{\mathcal{Y}^k_{K}} 1 \dy - \sum_{k=1}^N \int_{\mathcal{Y}^k_{K}} \rho(y) \dy \right| \\
    &\leqslant \sum_{k=1}^\nhyp \int_{\mathcal{Y}^k_{K}} |\rho(y)-\rho(z^k_{K})| \\
    &\leqslant L \sum_{k=1}^\nhyp \int_{\mathcal{Y}^k_{K}} \|y-z^k_{K}\| \mathrm{d} y\\
    &= \sum_{k=1}^{\nhyp} \mathcal{O}\left(\frac{1}{M^{d+1}}\right) \\
    &= \mathcal{O}\left(\frac{1}{M}\right)\,.
\end{align*}
Combining both, we obtain the desired result (note that $\int_{\mathcal{Y}} \rho(y) \dy = 1$).
$$\distortion_{\nhyp}^{\mathrm{H}} = \frac{d}{12}\frac{c^2}{M^2} + \mathcal{O}\left(\frac{1}{M^3}\right)\,.$$
We can replace $\nhyp$ in this formula. 
$$\distortion_\nhyp^{\mathrm{H}} =\frac{d}{12}\frac{c^2}{\nhyp^{2/d}} + \mathcal{O}\left(\frac{1}{\nhyp^{3/d}}\right)\,.$$
\end{proof}

While the above proof was carried out with $\cY=[0,c]^d$ with a volume of $\lambda(\cY)=c^d$, one can show the more general expression:
\begin{align}
    \distortion_\nhyp^{\mathrm{H}} =\frac{d}{12}\frac{\lambda(\cY)^{2/d}}{\nhyp^{2/d}} + \mathcal{O}\left(\frac{1}{\nhyp^{3/d}}\right)\,. \label{eq:general_histogram_risk}
\end{align}

Note that in the first order, the asymptotic quantization risk of the histogram does not depend on the distribution of the probability mass $\rho$ but only on the width of its support.

Both WTA and Histogram have the same asymptotic risk $\mathcal{O}\left(\nhyp^{-\frac{2}{d}}\right)$. However, the leading constant is always larger for Histogram (for all densities $\rho$) for large $\nhyp$. This means that WTA is strictly better than Histogram even in the asymptotic regime (Proposition \ref{th:leading_constant}). 

\begin{proposition}\label{th:leading_constant}
    Under the assumptions of Proposition \ref{th:histogram_risk}, the leading constant of $\mathcal{R}_\nhyp^{\mathrm{H}}$ is greater than that of $\mathcal{R}_\nhyp^{\mathrm{D-WTA}}$ for dimension $d\in\{1,2\}$ and for large $d$.  
\end{proposition}
\begin{proof}
    We show that the leading constant of $\mathcal{R}_\nhyp^{\mathrm{H}}$ is greater than that of $\mathcal{R}_\nhyp^{\mathrm{WTA}}$ for large $d$. Using Eq. (\ref{eq:general_histogram_risk}), we can write the following.
    \begin{align*}
        \frac{d}{12}\lambda(\cY)^{2/d} &> \frac{d}{2\pi e} \left(\int_\cY\dy\right)^{2/d} \\
        &= \frac{d}{2\pi e} \left(\int_\cY\dy\right)^{(2+d)/d} \frac{\int_{\mathcal{Y}}\rho(y) \dy}{\int_\cY \dy} & \text{(since $\int_\cY \rho(y)\dy=1$)}\\
        &= \frac{d}{2\pi e} \left(\int_\cY\dy\right)^{(2+d)/d} \frac{\int_{\mathcal{Y}} \left(\rho(y)^{d/(d+2)}\right)^{(d+2)/d} \dy}{\int_\cY \dy} & \text{(since $(x^r)^{1/r}=x$)}\\
        &\geq \frac{d}{2\pi e} \left(\int_\cY\dy\right)^{(2+d)/d} \left(\frac{\int_{\mathcal{Y}}\rho(y)^{d/(d+2)} \dy}{\int_\cY \dy}\right)^{(d+2)/d} & \text{(Jensen inequality applied to $x\mapsto x^{(d+2)/2}$)} \\
        &= \frac{d}{2\pi e} \left(\int_{\mathcal{Y}}\rho(y)^{d/(d+2)} \dy\right)^{(d+2)/d}\,.
    \end{align*}
    We recognize the leading constant of the Zador theorem on the last term (using the asymptotic equivalent $J_d\sim \frac{d}{2\pi e}$). The argument is the same for $d=1$ (in which case $J_d=\frac{1}{12}$), and $d=2$ (in which case $J_d=\frac{5}{18\sqrt{3}}<\frac{2}{12}$).
\end{proof}

\subsection{Discussion of the assumptions}
\label{sec:discussion}
We discuss in this section the validity of our assumptions.

\textbf{Boundedness assumption \ref{hyp:compact}}. The boundedness assumption is a necessary assumption to ensure that the volume of each cell is well-defined, such as in the case of a piece-wise uniform distribution. This is a customary assumption in theoretical analysis of K-Means algorithm \cite{emelianenko2008nondegeneracy}, as well as in centroidal Voronoi tessellations \cite{du1999centroidal}. Moreover, this assumption is valid in all realistic settings of pragmatic interest.

\textbf{Positivity assumption \ref{hyp:positive}}. The positivity assumption asserts that the data PDF is nonzero everywhere. This technical assumption is also reasonable for practical applications. For instance, if we assume that $\mathcal{X} \times \mathcal{Y}$ is bounded, modifying the distribution to $\tilde{\rho}(x,y) = (1+\varepsilon)^{-1} (\rho(x,y) + \varepsilon)$ with a sufficiently small $\varepsilon > 0$ does not alter the experimental results, and ensures that $\tilde{\rho} > 0$ everywhere.

\textbf{Lipschitz assumption \ref{hyp:lipschitz}}. The Lipschitz assumption of the data distribution was applied in Proposition \ref{th:histogram_risk} and \ref{th:leading_constant}. Note that no further assumption was made on the value of $L > 0$.

\textbf{Optimality assumption \ref{hyp:minimum}}. The assumption of quantization optimality is strong. There are two main arguments for its validity:
\begin{itemize}
\item The convergence of WTA has not been directly studied in the literature. However, we can see WTA as a \textit{conditional} gradient descent version of K-means \cite{pages2003optimal}. The convergence of K-means has been extensively studied \cite{sabin1986global, emelianenko2008nondegeneracy, bourne2015centroidal}, and some results could, as further work, be ported to WTA \cite{pages2003optimal}. Even though few results concern the convergence towards a global minimum of the quantization objective, theoretical evidence points toward the surprising effectiveness of this approach in most cases \cite{blomer2016theoretical}.
\item Our justification for Assumption \ref{hyp:minimum} is also empirical. We confirmed experimentally that the empirical quantization risk closely follows the theoretical optimal quantization risk given by Proposition \ref{th:wta_histogram_risk}. Indeed, we can see in Figure \ref{fig:quantitative} that the solid and dotted green curves of the quantization risk converge in the asymptotic regime, which is when the theoretical formula is valid. 
\end{itemize}
This gives us confidence that Assumption \ref{hyp:minimum} holds in practice, possibly in an approximate manner. Further research could explore in greater detail how the initialization of the Winner-Takes-All (WTA) training scheme influences the quality of the optimal solution. This investigation could build upon the findings of studies such as those by \citet{arthur2007k} and \citet{aggarwal2009adaptive}. 

\section{Experimental details}\label{sec:experimental_details}
\subsection{Synthetic data experiments}
\label{secapp:synthetic_data_experiments}
\subsubsection{Baselines.} 
\label{sec:syn-baselines}
Two standard conditional density estimation baselines were considered in our experiments: Mixture Density Networks (MDN)~\cite{bishop1994mixture} and the Histogram \cite{imani2018improving} mentioned in Section \ref{sec:quantization}.
We took care to assess all methods fairly, by using the same backbone architecture and number of hypotheses for all of them.
These baselines differ however from WTA-based methods on two main axes, output format and training loss, as explained hereafter:

\begin{itemize}
    \item \emph{MDN} uses a multi-head neural network to directly predict the parameters of a mixture of Gaussians. It is trained using the negative log-likelihood loss:
    \begin{equation}
        \mathcal{L}^{\mathrm{MDN}}(\theta) = - \log \hat{\rho}_{\theta}(\by \mid \bx)\,,
        \label{eq:mdn}
    \end{equation}
    where $\hat{\rho}_{\theta}(\by \mid \bx)$  is a mixture of Gaussians with parameters $\{\pi_k(\bx),\mu_k(\bx),\sigma_k(\bx)\}$, respectively denoting mixture weights, means and standard deviations. We considered isotropic Gaussians here. We enhanced the original MDN training scheme by incorporating the findings from \citet{brando2017mixture} to improve the numerical stability of the training. In particular, we employed the \emph{Log-sum-exp} trick \cite{blanchard2019accurate}, and modified the neural network's output to predict $(\mu,\log \sigma^2)$ rather than $(\mu,\sigma)$.
    \item \emph{\hist} is purely non-parametric, unlike MDN.
    It uses a multi-head neural network to predict scores $\gk(x) \in [0, 1]$ for each point in a fixed regular grid, defining the histogram bins.
    It is trained through backpropagation using the following loss: 
    \begin{equation}
        \mathcal{L}^{\mathrm{H}}(\theta) = - \log\gamma_{\theta}^{k^{\star}}(\bx) - \sum_{k \neq k^{\star}} \log(1 - \gk(\bx))\,,
    \end{equation}
    where $k^{\star} = \mathrm{argmin}_k \ell(\bg_k,\by)$ is the bin index in which the target falls. Note that the Histogram baseline can be seen as a specific instance of hypothesis-scores architecture used in the WTA setup, where each hypothesis is static and represents the central point of a histogram bin. In the context of the synthetic data experiments of Section \ref{sec:experiments}, the output space is $\mathcal{Y} = [-1,1]^2$ and we employed a regular grid defined by row $i \in [\![1,N_{\mathrm{rows}}]\!]$ and by column $j \in [\![1,N_{\mathrm{cols}}]\!]$. For each $x \in \cX$, if $k$ is the hypothesis index associated to bin $(i,j)$, we therefore have: $$f_{\theta}^{k}(x) = \left(-1 + \left(i - \frac{1}{2}\right) \frac{2}{N_{\mathrm{rows}}}, -1 + \left(j - \frac{1}{2}\right) \frac{2}{N_{\mathrm{cols}}}\right)\,,$$ with $K = N_{\mathrm{rows}} N_{\mathrm{cols}}$. In our comparisons described in Section \ref{sec:experiments}, we used $N_{\mathrm{rows}} = N_{\mathrm{cols}}$ in our experiments, except when $K = 20$ where we set $N_{\mathrm{rows}} = 5$ and $N_{\mathrm{cols}} = 4$.
\end{itemize}

\subsubsection{Architectures and training details} 

\textbf{Architectures.} In each training setup with synthetic data, we employed a two-hidden-layer multilayer perceptron. Each layer contained 256 hidden units and used ReLU activation functions. In the final layer, we utilized tanh activations for the hypotheses and sigmoid activations for the scores, where applicable. The last layer is duplicated depending on the number of outputs to produce: three times the number of modes in the case of MDN (mixture coefficients, means, and variances), twice the number of hypotheses in the WTA setup (scores and hypotheses are predicted), and the product of the number of rows and columns for the 2-dimensional histogram. Note that if the normalization of the scores is not inherently implemented in the architecture (\textit{e.g.,} with a softmax activation), they must be normalized when computing metrics by considering $\frac{\gamma_k(x)}{\sum_k \gamma_k(x)}$.

\textbf{Training details.} The Adam optimizer~\cite{kingma2014adam} was used with a constant learning rate of $0.001$ in each setup. The models were trained until convergence of the training loss, using early stopping to select the checkpoint for which the validation loss was the lowest. Each of the synthetic datasets consists of $100,000$  training points, and $25,000$ validation points. Each of the models was trained for $100$ epochs, with a batch size of $1024$.  

In each setup that involves WTA training, we used the compound loss $\mathcal{L}^{\mathrm{WTA}}+\beta \mathcal{L}^{\mathrm{scoring}}$ with $\beta = 1$. Note that we observed that the WTA training scheme leads to a fast convergence of the predictions $f_\theta(x)$, while the scoring heads $\gamma_\theta(x)$ are slightly slower to train. Indeed, each $\gamma^k_\theta(x)$ solves a binary classification task that evolves as the position of $f_\theta(x)$ is updated during training. Therefore, this objective is untractable at the beginning of the training, because the prediction $f_\theta(x)$ moves too quickly, and it only becomes feasible near the end of training, once the prediction has stabilized. This warrants further research on the scheduling of $\beta$ during training. 

\subsubsection{Metrics} 

For assessing the quality of the predictions, we used the following metrics for each input $x \in \cX$.
\begin{itemize}
    \item The Negative Log-Likelihood (NLL), which assesses the probabilistic quality of the predictions
    \begin{equation}
        \mathrm{NLL}(\hpx,\rho_x) = - \int_\cY\log \hpyx\pyx \mathrm{d}y\,,
        \label{eq:nll}
    \end{equation} 
    where $\hpyx$ is the estimated density, which is assumed to integrate to 1.
    \item The Earth Mover's Distance (EMD): 
    \begin{equation}
        \mathrm{EMD}(\hpx,\rho_x) = \min _{\psi \in \Psi} \sum_{\by_s \sim \px} \sum_{\hat{\by}_k \sim \hat{\rho}_x} \psi_{s, k} \lVert \by_s - \hat{\by}_k \rVert\,,  
        \label{eq:emd}
    \end{equation}
    where $\psi \in \Psi$ is a transport plan belonging to the set of valid transport plans \cite{kantorovich1942translocation}. 
    \item The Quantization Error, as defined in \citet{pages2003optimal}:
\begin{align}
    \distortion(\mathcal{Z}) = \int_\cY \min_{z\in\cZ} \|y-z\|^2 \;\rho_x(\by) \mathrm{d}\by\,.
    \label{eq:quantization}
\end{align} 
\end{itemize}

Note that equations \eqref{eq:nll}, \eqref{eq:emd} and \eqref{eq:quantization} assume that the target distribution $\rho_x$ is available. However, this is not usually the case for real-world tasks, for which typically one sample $y \sim \rho_x$ is available for each input $x$. In this case, the EMD loses its interpretation as a distance between distributions. Nevertheless, the NLL and Distortion errors can still be computed on an average basis over the test set, with the NLL and Quantization Errors defined as 
\begin{align}
\mathrm{NLL} &= -\frac{1}{N} \sum_{i=1}^{N} \log \hat{\rho}_{x_i}(y_i)\,,\label{eq:nll_empirical}\\
\distortion &= \frac{1}{N} \sum_{i=1}^{N} \min_{z \in \mathcal{Z}_{i}} \|y_i - z\|^2\,,\label{eq:distortion_empirical}
\end{align}
where $N$ is the number of pairs $(x_i,y_i)$ in the test set, and $\mathcal{Z}_{i} = \left\{f_{\theta}^l(x_i)\right\}_{l \in \bn}$.  

\subsubsection{Evaluation details}

The results of Figure \ref{fig:quantitative} were computed according to the following details. Note that each evaluation was performed with $N = 2,000$ test points. The results are averaged over three random seeds (see Figure \ref{fig:stds_results}). 

\textbf{NLL computation.} The NLL was computed following \eqref{eq:nll_empirical}, with $N$ the number of points on each test set, and $\hat{\rho}_{x}$ is for instance given in \eqref{eq:voronoi-wta} in Voronoi-WTA. The Volume 
\begin{equation}
V(f_{\theta}^k(x),K_h) = \int_{\mathcal{Y}_k(x)} K_h(f_{\theta}^k(x),\tilde{y})\mathrm{d}\tilde{y}\,,
\label{app:volume}
\end{equation}
was computed with the normalized Gaussian kernel: 
\begin{equation}K_h(f_{\theta}^k(x),y) = \frac{1}{(2 \pi)^{\frac{d}{2}} h^{d}} \mathrm{exp} \left( \dfrac{\lVert y - f_{\theta}^k(x) \rVert^{2}}{2 h^2} \right)\,,
\label{eq:normalized_gaussian_kernel}
\end{equation}
where $d = \mathrm{dim}(\mathcal{Y})$. In particular, $d = 2$ for the synthetic data experiments and $d = 1$ for the UCI datasets experiments.

In practice, the Volume \eqref{app:volume} can be computed efficiently, rewriting each $\tilde{\by}$ in the integral as
\begin{equation} \label{eq:change-var}
    \tilde{\by} = \fk(\bx) + t \bs\,, 
\end{equation}
where $\bs \in \mathbb{S}^{d-1}$ is a direction on the unit sphere and ${t \in [0,l_{\fk(\bx)}(\bs)]}$ a scalar step. Here
$l_{\fk(\bx)}(\bs)$ is the so-called \emph{directional radius}, defined as the maximum $u \in \mathbb{R}_+$ such that $\fk(\bx) + u \bs \in \Ykx$ if it exists, and $l_{\fk(\bx)}(\bs) = \infty$ otherwise. In practice, we computed each directional radius in $\mathcal{O}(\nhyp d)$ operations by following Equations 7 and 8 from \citet{polianskii2022voronoi}, leveraging the structure of Voronoi tesselation as detailed in \citet{polianskii2019voronoi}. The case of unbounded Voronoi cells did not appear here since the output is restricted to the square $[-1,1]^2$.   
As explained by \citet[Sec.3.1]{polianskii2022voronoi}, \eqref{eq:change-var} allows writing $V(\fk(\bx), K_h)$ as a double-integral in spherical coordinates:
$$V(f_{\theta}^k(x),y)=\int_{s \in \mathbb{S}^{d-1}} \int_{t \in \left[0, l_{f_{\theta}^k(x)}(s)\right]} K(\frac{t}{h} s) t^{d-1} \mathrm{~d} t \mathrm{~d}s\,,$$
where: $K(x) \triangleq \mathrm{exp}(-\frac{\|x\|^{2}}{2})$.

As the inner integral has a closed-form solution for the Gaussian kernel \cite{polianskii2022voronoi}, a Monte Carlo approximation of the outer integral allows us to estimate $V(\fk(\bx), K_h)$:
\begin{equation}
V(f_{\theta}^k(x),y) \simeq \frac{2 \pi^{\frac{d}{2}}}{N^{\prime} \Gamma\left(\frac{n}{2}\right)} \sum_{j = 1}^{N^{\prime}} \int_{\left[0, l_{f_{\theta}^k(x)}(s_j)\right]} K(\frac{t}{h} s_j) t^{d-1} \mathrm{~d} t
\simeq \frac{1}{N^{\prime}} \sum_{j=1}^{N^{\prime}} (2\pi h^{2})^{\frac{d}{2}} \bar{\gamma}\left(\frac{d}{2}, \frac{l_{f_{\theta}^k(x)}(s_j)^{2}}{2 h^{2}}\right)\,,
\label{eqapp:mc_approx}
\end{equation}

where $\Gamma$ is the gamma function, $N^{\prime}$ is the number of points $\{s_j\}$ sampled on the unit sphere $\mathbb{S}^{d-1}$ (or \textit{versors}) and $\bar{\gamma}(a, z) \triangleq \frac{1}{\Gamma(a)} \int_0^z t^{a-1} e^{-t} \mathrm{~d} t$ is the incomplete gamma function. When $d=2$, \eqref{eqapp:mc_approx} simplifies to:  
$$V(f_{\theta}^k(x),y) \simeq \frac{1}{N^{\prime}} \sum_{j=1}^{N^{\prime}} 2 \pi h^{2} \left(1-\mathrm{exp}\left(- \frac{l_{f_{\theta}^k(x)}(s_j)^{2}}{2 h^{2}}\right)\right)\,.$$
In practice, we used $N^{\prime} = 40$ for our experiments. 

\textbf{EMD computation.} The EMD was computed as $$\mathrm{EMD} = \frac{1}{N} \sum_{i=1}^{N} \mathrm{EMD}(\hat{\rho}_{x_i},\rho_{x_i})\,,$$
where $N = 2,000$ and $\mathrm{EMD}(\hat{\rho}_{x},\rho_{x})$ is defined in \eqref{eq:emd}. Computing the EMD for each input requires sampling from both the predicted and target distributions, with the assumption that sampling from the target distribution is feasible. For the predicted distribution, a rejection sampling procedure was implemented. This involved initially selecting a cell $k$ based on the distribution of scores $\left\{ \gamma_{\theta}^l(x) \right\}$.
Samples were then repeatedly drawn from the distribution $K_h(f_{\theta}^k(x),\cdot)$ until a sample falling within the cell $\mathcal{Y}_{\theta}^k(x)$ was obtained. 
In our experiments, we matched empirical measures by taking $500$ samples each from both the predicted and target distributions.
Additionally, to reduce computational complexity, especially when dealing with a large number of hypotheses, one can employ the hit-and-run sampling technique as outlined in Alg. 2 by \citet{polianskii2022voronoi}.   

\textbf{Theoretical curves.} The theoretical curves of Figure \ref{fig:quantitative} were computed according to \eqref{eq:zador_risk} and \eqref{eq:histogram_risk}. The calculation of \eqref{eq:zador_risk} was performed through Monte-Carlo integration using $10,000$ samples across the output space for each input $x$ (averaging over $10$ inputs here), for $K \in \{9, 16, 25, 49, 100\}$. 
The target density $\rho$ is explicitly defined in the \textit{Changing Damier} and \textit{Single Gaussian} datasets, while in the \textit{Rotating Moons} dataset, it was approximated through Kernel Density Estimation (with Gaussian kernel and bandwidth of $0.2$). 

\subsubsection{Additional results}
\label{app:additional_results_synth}

\begin{figure}
\centering
    \includegraphics[width=0.95\linewidth]{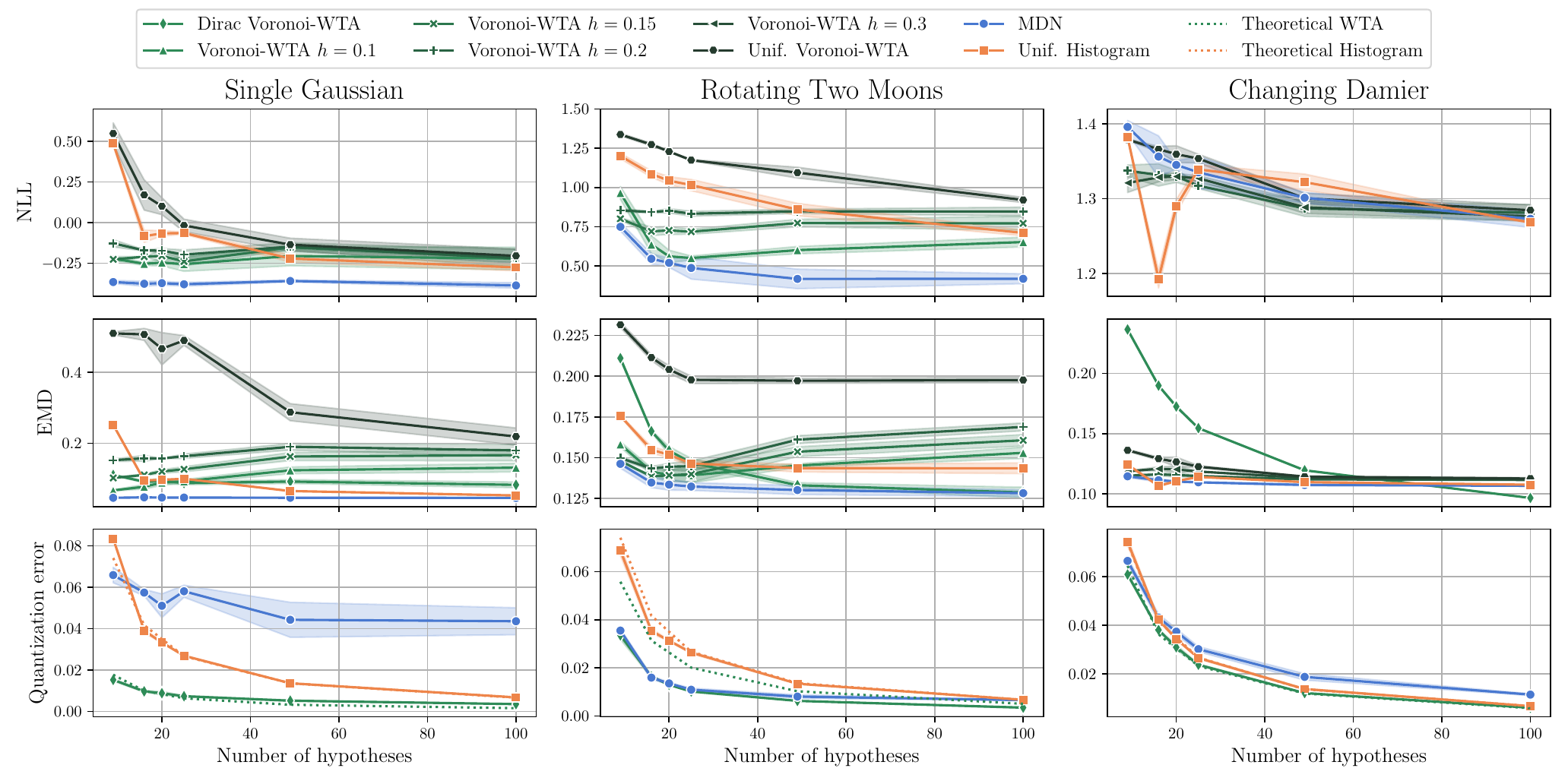}
    \caption{\textbf{Standard deviations across three random seeds in the results of Figure \ref{fig:quantitative}.} To simplify the presentation, only the following Voronoi-WTA curves are displayed: $h = 0.1, 0.15, 0.2$ for the Single Gaussian and Rotating Moons datasets, and $h = 0.2, 0.3$ for the Changing Damier dataset. Additionally, the y-axis of the EMD plot for the Changing Damier was cropped in Figure \ref{fig:quantitative} to enhance the scale readability.
    }
    \label{fig:stds_results}
\end{figure}

\begin{figure}
    \centering
\includegraphics[width=0.92\linewidth]{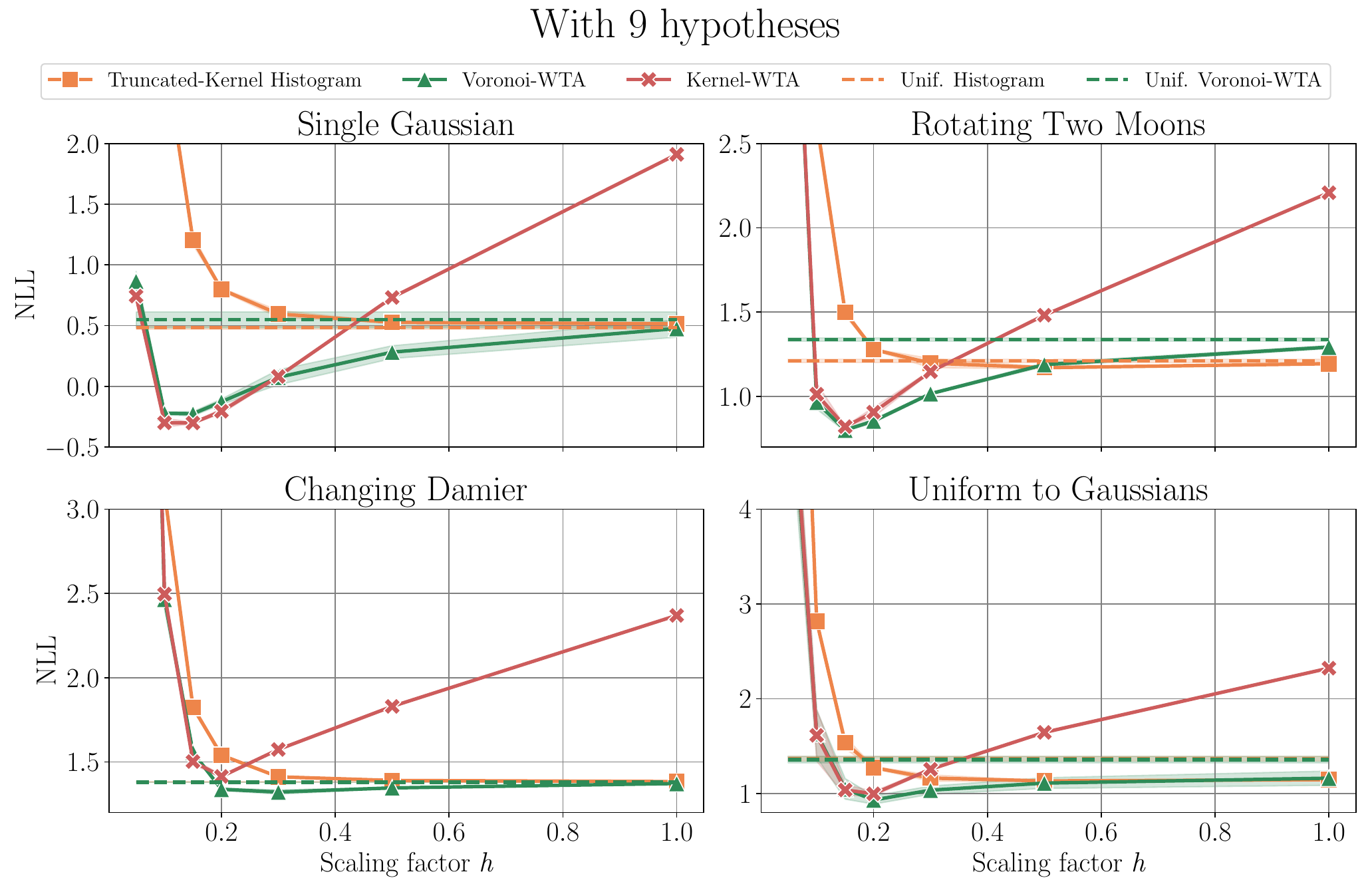}
    \caption{\textbf{NLL \textit{vs. h} with 9 hypotheses.}}
    \label{fig:nll_vs_h_9_hyps_multiple_plots}
\end{figure}

\begin{figure}
    \centering
\includegraphics[width=0.92\linewidth]{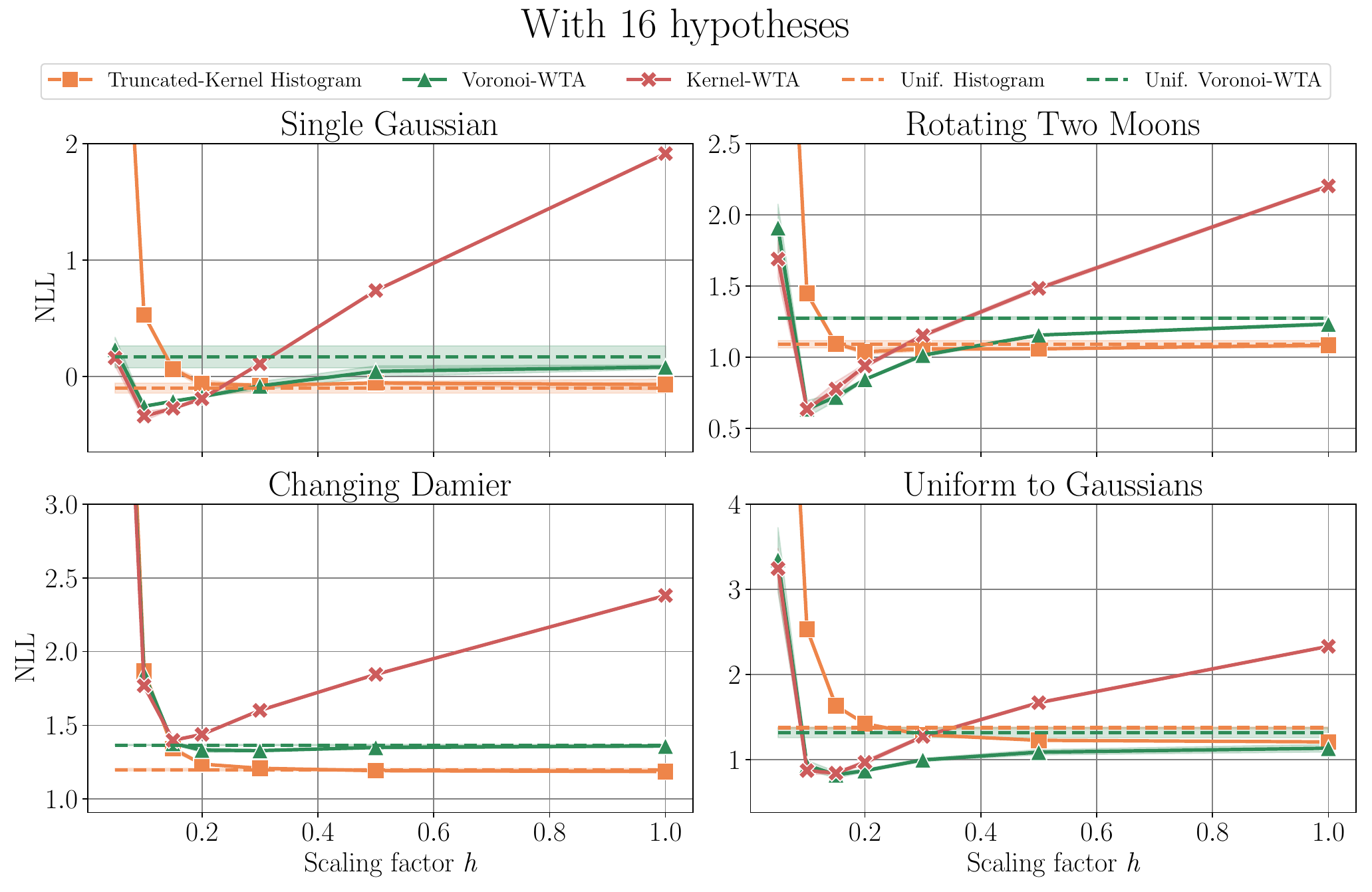}
    \caption{\textbf{NLL \textit{vs. h} with 16 hypotheses.}}
    \label{fig:nll_vs_h_16_hyps_multiple_plots}
\end{figure}

\begin{figure}
    \centering
\includegraphics[width=0.92\linewidth]{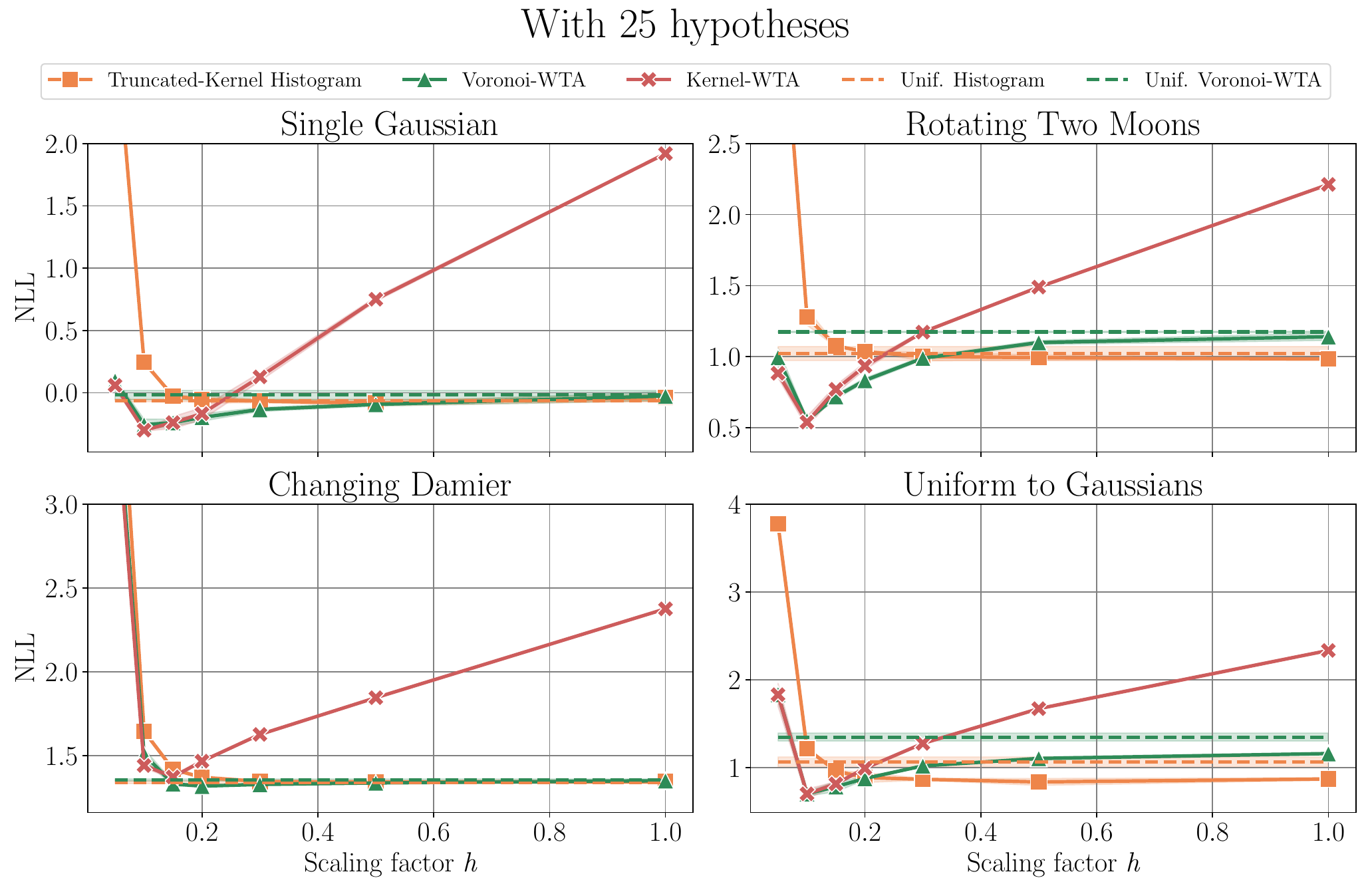}
    \caption{\textbf{NLL \textit{vs. h} with 25 hypotheses.}}
    \label{fig:nll_vs_h_25_hyps_multiple_plots}
\end{figure}

\begin{figure}
    \centering
\includegraphics[width=0.92\linewidth]{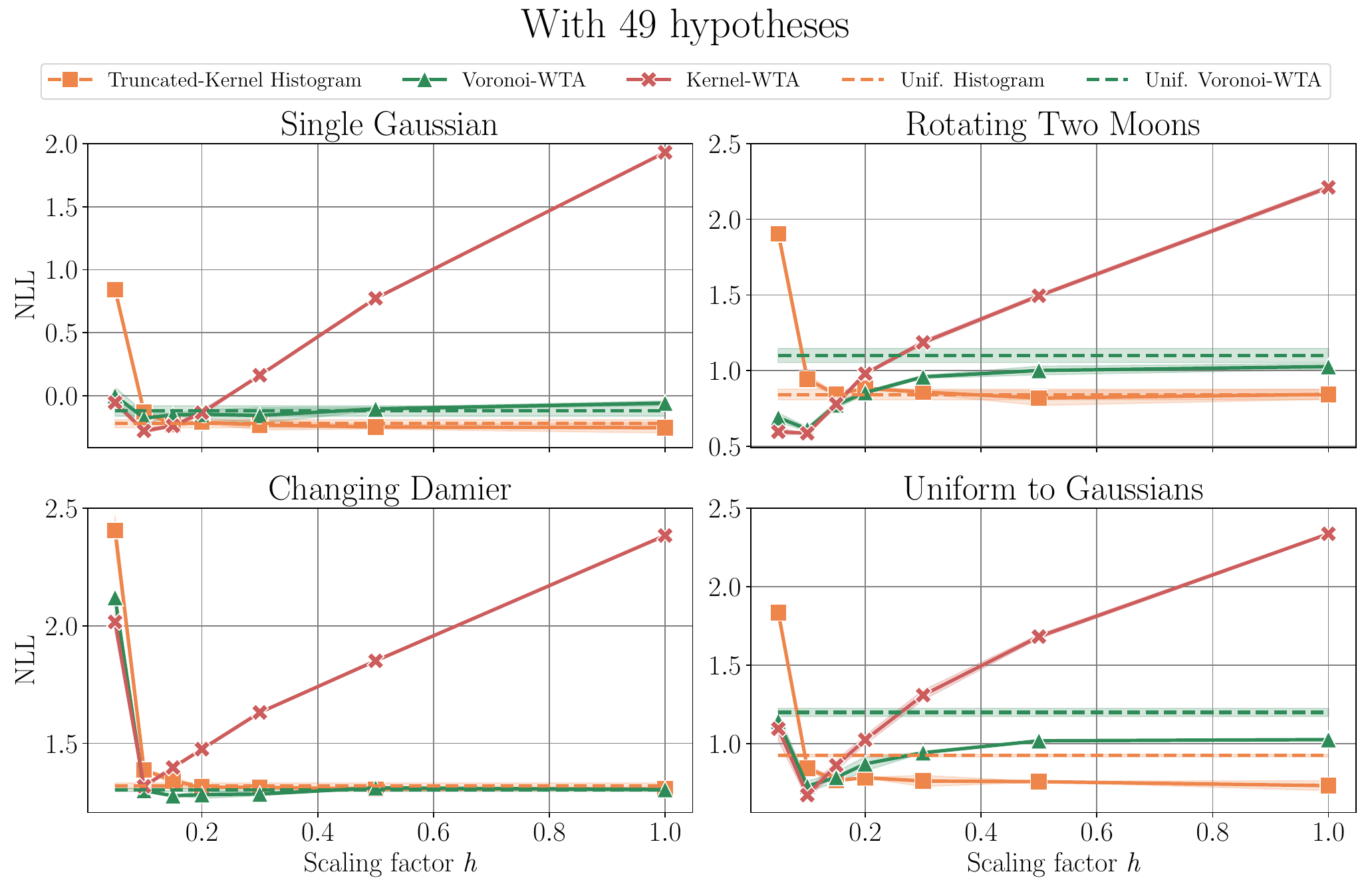}
   \caption{\textbf{NLL \textit{vs. h} with 49 hypotheses.}}
   \label{fig:nll_vs_h_49_hyps_multiple_plots}
\end{figure}

Additional results on the four synthetic datasets \textit{Rotating Two Moons}, \textit{Changing Damier}, \textit{Single Gaussian} and \textit{Uniform to Gaussians} presented in Section \ref{sec:exp-setting} are provided in Figures \ref{fig:nll_vs_h_9_hyps_multiple_plots}, \ref{fig:nll_vs_h_16_hyps_multiple_plots}, \ref{fig:nll_vs_h_25_hyps_multiple_plots}, \ref{fig:nll_vs_h_49_hyps_multiple_plots}. As in Figure \ref{fig:NLLvsh}, the aim here is to demonstrate the resilience of Voronoi-WTA with respect to the choice of the scaling factor $h$ in comparison with different baselines. These include Kernel-WTA, and Truncated-Kernel Histogram, where truncated kernels were placed on the bin position. Here, the truncated kernel variants were computed with normalized Gaussian kernels \eqref{eq:normalized_gaussian_kernel} with $d=2$ here. Uniform kernels were used in the baselines `Unif. Histogram' and `Unif. Voronoi-WTA'. 

At first, several sanity checks can be carried out. One can verify that as $h \rightarrow 0$, Kernel-WTA and Voronoi-WTA are equivalent in all datasets. This outcome is anticipated since, in such regimes, the impact of truncation is virtually negligible, as discussed in Section \ref{sec:truncated-kernels}. Similarly, when $h$ approaches infinity, both Voronoi-WTA and the Truncated-Kernel Histogram tend to align with their respective uniform versions -- Uniform Voronoi-WTA and Uniform Histogram. This behavior is expected due to the bounded nature of the output space. Note that we empirically observe a convex U-shape for the NLL curves on the validation set of Kernel-WTA and Voronoi-WTA as a function of $h$, allowing us to use adaptive grid search algorithms, such as the Golden Section Search \cite{kiefer1953sequential}, which have a fast convergence rate in this setting. 

Additionally, observations consistent with those discussed in Section \ref{fig:NLLvsh} have been made. 
These findings highlight the robustness of Voronoi-WTA against variations in the $h$ parameter, demonstrating its superior performance over Kernel-WTA, at larger values of $h$.
In contrast, Voronoi-WTA's advantages over the Truncated-Kernel Histogram become more pronounced at smaller $h$ values.
Except for the Changing Damier dataset, where the Histogram shows an immediate advantage, the results align with the analysis of Section \ref{sec:quantitative-results}, indicating that the Histogram's performance suffers from suboptimal hypothesis placement in this regime. However, it is important to note that as $h$ increases, the Truncated-Kernel Histogram sometimes matches or exceeds the performance of other methods at a fixed $h$. Likewise, when $\nhyp$ is large, the Histogram achieves greater resolution and becomes competitive, thus compensating the naive placement of the hypotheses.

The slight deterioration in performance when considering uniform kernels compared with Gaussian kernels is attributed to the so-called 
`compactification' issue, as discussed in \citet{polianskii2022voronoi}. Indeed, in scenarios with large cells in the Voronoi tesselation,\footnote{See for instance the hypotheses placement of a WTA-based trained model on a Gaussian distribution in Figure \ref{fig:qualitative}} the application of a uniform kernel can significantly worsen NLL results. This effect is indeed less pronounced in the histogram method, where the volume of each cell remains constant.

Consistently with observations made from audio data discussed in Section \ref{sec:audio}, Figures \ref{fig:nll_vs_h_9_hyps_multiple_plots}, \ref{fig:nll_vs_h_16_hyps_multiple_plots}, \ref{fig:nll_vs_h_25_hyps_multiple_plots}, and \ref{fig:nll_vs_h_49_hyps_multiple_plots} illustrate that the performance gap between Kernel-WTA and Voronoi-WTA widens as $h$ increases and the number of hypotheses grows. It is important to acknowledge a limitation: in scenarios where the hypotheses are few and sufficiently spaced apart, the impact of truncation becomes minor. Under these conditions, Kernel-WTA is likely to offer comparable performance to Voronoi-WTA for a wide range of $h$ values.

\subsection{UCI datasets}
\label{app:uci-datasets}
We conducted additional experiments on the UCI Regression Datasets \cite{ucidataset}, which are a standard benchmark to evaluate conditional density estimators. Table \ref{tab:datasets-description} provides the sizes of the datasets. 

\textbf{Experimental setup.} All the estimators mentioned in our manuscript (Mixture Density Network, Histogram-based methods, and WTA-based methods) are trained and evaluated on these datasets, and results are provided in Table \ref{tab:uci-nll} and Table \ref{tab:uci-rmse}. The results were computed following the same experimental protocol from \citet{hernandez2015probabilistic}. In particular, each dataset is divided into 20 train-test folds, except the protein dataset, which is divided into 5 folds, and the Year Prediction MSD dataset for which a single train-test split is used. Moreover, we use the same neural network backbone for each baseline: a one-hidden layer MLP with ReLU activation function, containing 50 hidden units except for the Protein and Year datasets, for which 100 hidden units were utilized.
Each model was trained using the Adam optimizer over $1,000$ epochs with a constant learning rate of $0.01$.
Our data loading pipeline for the UCI datasets was adapted from the open-sourced implementation of \citet{han2022card}. In the results presented in Table \ref{tab:uci-nll} and Table \ref{tab:uci-rmse}, we follow the convention of highlighting the best models for each dataset in bold, based on the mean value of the metrics. Additionally, any model whose confidence interval overlaps with this best mean is also bolded.

\textbf{Baselines.} Those tables include results from three baselines reported from Table 1 of \citet{lakshminarayanan2017simple} which we use as references for those benchmarks: `PBP' stands for Probabilistic Back Propagation \cite{hernandez2015probabilistic}, and `MC-dropout' corresponds to Monte Carlo Dropout \cite{gal2016dropout}.
The Histogram NLL was computed with truncated kernels (TK-NLL for Truncated-Kernel-Histogram), following the same tuning protocol for $h$ as Voronoi-WTA (V-WTA) and Kernel-WTA (K-WTA).
MDN corresponds to a mixture density network with Gaussian kernels as in Appendix \ref{sec:syn-baselines}. The multi-hypotheses baselines (MDN, Histogram, and WTA-based methods) were trained with $K = 5$. 
In this setup, the regular grid of the histogram was defined with $f_{\theta}^k(x) = \frac{2(k-3)}{5}$ for $k \in \{1,\dots,5\}$.

\textbf{Metrics.} The computed metrics correspond to first the RMSE, which is defined as $\mathrm{RMSE} = \sqrt{\frac{1}{N} \sum_{i} \ell(\hat{y}_i, y_i)}$, where $\hat{y}_i$ denotes the estimated conditional mean, which was estimated with $\sum_{k=1}^{\nhyp} \gamma_k(x) z_k(x)$ for the WTA variants, and $\sum_{k=1}^{\nhyp} \pi_k(x) \mu_k(x)$ for MDN. NLL has been calculated in the same way as in the previous sections, with $\mathrm{dim}(\mathcal{Y}) = 1$.

\textbf{Evaluation details.} Our density estimators were trained according to the procedures outlined in Section \ref{sec:experiments}. Post-training, the scaling factor $h$ was tuned based on the average NLL over the validation set (20 \% of the training data) using a golden section search \cite{kiefer1953sequential} (with tolerance set to  $0.1$ and the search interval bounded by $[0.1, 2]$), following a similar protocol as \citet{gal2016dropout}. In this setup, as per the guidelines in \citet{hernandez2015probabilistic}, we normalized both input and output variables for training using the means and standard deviations from the training data. For evaluation, we restored the original scale of the output predictions with the transformation $f_{\theta}^k(x) \mapsto \mu_{\mathrm{train}} + \sigma_{\mathrm{train}} f_{\theta}^k(x)$ where $\mu_{\mathrm{train}}$ and $\sigma_{\mathrm{train}}$ represent the empirical mean and standard deviation of the response variable across the training set. This transformation also applies to the predicted means of the MDN, while the predicted standard deviations were simply multiplied by $\mu_{\mathrm{train}}$. 

Results confirm the claims made in our manuscript and provide the following insights: 

\begin{itemize}
    \item Voronoi-WTA outperforms the Histogram-based estimator in terms of NLL and performs comparably to Kernel-WTA. This aligns with the findings discussed in \ref{app:additional_results_synth}, where Voronoi-WTA's edge over Kernel-WTA is less pronounced in settings with a limited number of hypotheses and where $h$ has been already optimized.
    Future research will examine how accurately optimizing $h$ during validation affects performance when there is a distribution shift between validation and test samples. Such a study could specifically assess potential performance discrepancies between Voronoi-WTA and Kernel-WTA under these conditions.
    \item WTA-based estimators outperform the Histogram-based estimator in terms of RMSE.
    \item Voronoi-WTA is competitive with MDN for NLL. This is a promising result as the NLL is optimized only during validation for Voronoi-WTA, and during training for MDN. Moreover, we faced stability issues when training MDN (\textit{e.g.,} numerical overflows in log-likelihood computation), that we did not encounter with Voronoi-WTA.
\end{itemize}

Looking at Table \ref{tab:uci-nll} and Table \ref{tab:uci-rmse}, we can see that for data-intensive tasks (Protein, Year) Voronoi-WTA is on par with standard baselines in those benchmarks, and occasionally outperforms them, both in terms of NLL and RMSE.
However, for tasks with limited data available, it seems that WTA is not the most suitable method. This underperformance in the small data regime could be expected: because of the competitive nature of the WTA training scheme, each prediction model only sees a fraction of the data.

\begin{table}
    \caption{\textbf{UCI Regression benchmark datasets.} $N$ is the number of data samples and $\mathrm{dim}(\mathcal{X})$ the input dimension. Here, the output space $\mathcal{Y}$ is one dimensional.}
    \resizebox{\textwidth}{!}{
    \begin{tabular}{l|cccccccccc}
    \toprule \text {Dataset } & \text {Boston} & \text {Concrete} & \text {Energy} & \text { Kin8nm } & \text {Naval} & \text {Power} & \text {Protein} & \text {Wine} & \text {Yacht} & \text {Year} \\
\midrule$\left(N; \mathrm{dim}(\mathcal{X})\right)$ & (506;13) & (1030;8) & (768;8) & (8192;14) & (11934;16) & (9568;4) & (45730;9) & (1599;11) & (308;6) & (515345;90)\\
\bottomrule
\end{tabular}}
    \label{tab:datasets-description}  
\end{table}

\begin{table}[h]
    \centering
    \caption{\textbf{UCI regression benchmark datasets comparing NLL with 5 hypotheses.} $^{\star}$ corresponds to reported results from \citet{lakshminarayanan2017simple}. `--' corresponds to cases where MDN has not converged. Best results are in \textbf{bold}. $\pm$ represents the standard deviation over the splits (non-applicable for the year dataset).}
    \resizebox{\columnwidth}{!}{
    \begin{tabular}{l  ccc|cccc}
    \toprule
     & \multicolumn{7}{c}{NLL ($\downarrow$)}\\
     \cmidrule(l{2pt}r{2pt}){2-8}
    Datasets & PBP$^{\star}$ & MC Dropout$^{\star}$ & Deep Ensembles$^{\star}$ & MDN & TK-Hist & K-WTA & V-WTA\\
    \midrule
    Boston & \textbf{2.57 $\pm$ 0.09} & \textbf{2.46 $\pm$ 0.25} & \textbf{2.41 $\pm$ 0.25} & 2.95 $\pm$ 0.31 & 2.83 $\pm$ 0.17 & \textbf{2.48 $\pm$ 0.16} & \textbf{2.48 $\pm$ 0.19} \\
Concrete & 3.16 $\pm$ 0.02 & \textbf{3.04 $\pm$ 0.09} & \textbf{3.06 $\pm$ 0.18} & 3.96 $\pm$ 0.24 & 3.47 $\pm$ 0.12 & \textbf{3.09 $\pm$ 0.10} & \textbf{3.08 $\pm$ 0.12} \\
Energy & 2.04 $\pm$ 0.02 & 1.99 $\pm$ 0.09 & \textbf{1.38 $\pm$ 0.22} & \textbf{1.25 $\pm$ 0.25} & 2.39 $\pm$ 0.11 & 2.27 $\pm$ 1.22 & 2.22 $\pm$ 1.20 \\
Kin8nm & -0.90 $\pm$ 0.01 & -0.95 $\pm$ 0.03 & \textbf{-1.20 $\pm$ 0.02} & -0.87 $\pm$ 0.05 & -0.70 $\pm$ 0.03 & -0.73 $\pm$ 0.03 & -0.85 $\pm$ 0.05 \\
Naval & -3.73 $\pm$ 0.01 & -3.80 $\pm$ 0.05 & \textbf{-5.63 $\pm$ 0.05} & \textbf{-5.47 $\pm$ 0.29} & -3.06 $\pm$ 0.02 & -1.94 $\pm$ 0.00 & -3.52 $\pm$ 0.38 \\
Power & 2.84 $\pm$ 0.01 & \textbf{2.80 $\pm$ 0.05} & \textbf{2.79 $\pm$ 0.04} & 3.02 $\pm$ 0.07 & 3.18 $\pm$ 0.02 & \textbf{2.81 $\pm$ 0.05} & \textbf{2.85 $\pm$ 0.06}\\
Protein & 2.97 $\pm$ 0.00 & 2.89 $\pm$ 0.01 & 2.83 $\pm$ 0.02 & -- & 2.64 $\pm$ 0.01 & \textbf{2.39 $\pm$ 0.03} & \textbf{2.42 $\pm$ 0.04} \\
Wine & 0.97 $\pm$ 0.01 & 0.93 $\pm$ 0.06 & 0.94 $\pm$ 0.12 & \textbf{-1.53 $\pm$ 0.76} & 0.46 $\pm$ 0.10 & 0.42 $\pm$ 0.18 & 0.37 $\pm$ 0.17 \\
Yacht & 1.63 $\pm$ 0.02 & 1.55 $\pm$ 0.12 & \textbf{1.18 $\pm$ 0.21} & 2.43 $\pm$ 0.72 & 2.80 $\pm$ 0.23 & 2.23 $\pm$ 0.52 & 2.05 $\pm$ 0.46 \\
Year & 3.60 $\pm$ NA & 3.59 $\pm$ NA & 3.35 $\pm$ NA & -- & 3.57 $\pm$ NA & \textbf{3.26 $\pm$ NA} & 3.29 $\pm$ NA \\

\bottomrule
    \end{tabular}
    }
    \label{tab:uci-nll}
    \end{table}

\begin{table}[h]
    \centering
    \caption{\textbf{UCI regression benchmark datasets comparing RMSE with 5 hypotheses.} $^{\star}$ corresponds to reported results from \citet{lakshminarayanan2017simple}. `--' corresponds to cases where MDN has not converged. Best results are in \textbf{bold}. $\pm$ represents the standard deviation over the splits (non-applicable for the year dataset).}
    \resizebox{\columnwidth}{!}{
    \begin{tabular}{l  ccc|cccc}
    \toprule
    & \multicolumn{7}{c}{RMSE ($\downarrow$)}\\
     \cmidrule(l{2pt}r{2pt}){2-8}
    Datasets & PBP$^{\star}$ & MC Dropout$^{\star}$ & Deep Ensembles$^{\star}$ & MDN & TK-Hist & K-WTA & V-WTA\\
    \midrule
    Boston & \textbf{3.01 $\pm$ 0.18} & \textbf{2.97 $\pm$ 0.85} & \textbf{3.28 $\pm$ 1.00} & \textbf{3.65 $\pm$ 1.15} & 5.53 $\pm$ 1.26 & \textbf{3.54 $\pm$ 1.16} & \textbf{3.54 $\pm$ 1.16} \\
Concrete & \textbf{5.67 $\pm$ 0.09} & \textbf{5.23 $\pm$ 0.53} & 6.03 $\pm$ 0.58 & 7.52 $\pm$ 0.96 & 9.03 $\pm$ 0.68 & 6.02 $\pm$ 0.65 & 6.02 $\pm$ 0.65 \\
Energy & 1.80 $\pm$ 0.05 & \textbf{1.66 $\pm$ 0.19} & 2.09 $\pm$ 0.29 & 2.35 $\pm$ 0.45 & 3.89 $\pm$ 0.48 & 2.53 $\pm$ 0.99 & 2.53 $\pm$ 0.99 \\
Kin8nm & 0.10 $\pm$ 0.00 & 0.10 $\pm$ 0.00 & 0.09 $\pm$ 0.00 & \textbf{0.08 $\pm$ 0.00} & 0.14 $\pm$ 0.01 & 0.10 $\pm$ 0.01 & 0.10 $\pm$ 0.01 \\
Naval & 0.01 $\pm$ 0.00 & 0.01 $\pm$ 0.00 & \textbf{0.00 $\pm$ 0.00} & \textbf{0.00 $\pm$ 0.00} & 0.01 $\pm$ 0.00 & \textbf{0.00 $\pm$ 0.00} & \textbf{0.00 $\pm$ 0.00} \\
Power & \textbf{4.12 $\pm$ 0.03} & \textbf{4.02 $\pm$ 0.18} & \textbf{4.11 $\pm$ 0.17} & \textbf{4.11 $\pm$ 0.18} & 7.55 $\pm$ 0.17 & \textbf{4.18 $\pm$ 0.16} & \textbf{4.18 $\pm$ 0.16} \\
Protein & 4.73 $\pm$ 0.01 & \textbf{4.36 $\pm$ 0.04} & 4.71 $\pm$ 0.06 & -- & 4.47 $\pm$ 0.02 & \textbf{4.39 $\pm$ 0.10} & \textbf{4.39 $\pm$ 0.10} \\
Wine & \textbf{0.64 $\pm$ 0.01} & \textbf{0.62 $\pm$ 0.04} & \textbf{0.64 $\pm$ 0.04} & \textbf{0.65 $\pm$ 0.04} & 0.67 $\pm$ 0.04 & \textbf{0.63 $\pm$ 0.04} & \textbf{0.63 $\pm$ 0.04} \\
Yacht & \textbf{1.02 $\pm$ 0.05} & \textbf{1.11 $\pm$ 0.38} & 1.58 $\pm$ 0.48 & 4.08 $\pm$ 1.57 & 8.27 $\pm$ 2.83 & 3.28 $\pm$ 1.39 & 3.28 $\pm$ 1.39 \\
Year & 8.88 $\pm$ NA & \textbf{8.85 $\pm$ NA} & 8.89 $\pm$ NA & -- & 9.31 $\pm$ NA & 9.09 $\pm$ NA & 9.09 $\pm$ NA \\

\bottomrule
        \end{tabular}
        }
        \label{tab:uci-rmse}
        \end{table}

\subsection{Audio data experiments}
\label{app:audio}
\subsubsection{Setup}
\label{app:setup}
This section describes in greater detail the experimental setup from Section \ref{sec:audio}. The audio data experiments are based on the protocol of \citet{schymura2021pilot, letzelter2023resilient} which is given as follows. Nevertheless, distinctions are to be made with respect to previous works. While \citet{letzelter2023resilient} study the case of punctual sound source localization \cite{grumiaux2022survey}, we apply our estimators here the more general problem of data uncertainty quantification, due for instance to an actual spatial dispersion of sound sources or annotation errors. This extension is described in greater detail in the paragraph `Synthetic perturbations' given below.

\textbf{Dataset preprocessing.} In our experiments, we used the ANSYN dataset \cite{adavanne2018sound}, which contains spatially localized sound events under anechoic conditions. We conformed to the dataset processing techniques as detailed in works by \citet{schymura2021pilot, letzelter2023resilient}. We employed the first-order Ambisonics format with four input audio channels. The audio recordings, with a 44.1 kHz sampling rate, were segmented into 30-second durations. These segments were further divided into non-overlapping chunks of 2 s to serve as the basis for training data. Spectrograms were calculated with a Hann window of 0.04 seconds for the Short Term Fourier Transform calculations. This was done with a 50\% overlap between frames and utilizing 2048 points for the Fast Fourier Transform computation. The information input into the models included both the amplitude and phase data, stacked channel-wise.

\textbf{Architecture.} We utilized SeldNet \cite{adavanne2018sound} as backbone (with $\sim 1.6$ M parameters). The data processing starts with the preprocessing of raw audio, which is then fed into the model in the form of spectrograms of a set duration, including phase information. The model then provides localization outputs at the specified resolution, in this case, considering $T=100$ output time steps for each segment. The architecture processes the data through feature extraction modules, including Convolutional Neural Networks (CNNs) and Bi-directional Gated Recurrent Unit (GRU) layers, creating a representation for each time step at the determined output resolution. These intermediate representations are subsequently connected to the final localization predictions via Fully Connected (FC) layers. To suit the Winner-takes-all framework, the terminal FC layers are divided into $K$ separate FC heads, each delivering a two-dimensional output (azimuth and elevation) at each time step. 
Additionally, the system incorporates score heads at the final stage, each yielding a single value between 0 and 1, achieved through a sigmoid activation function. Note that the `Histogram' baseline in the audio experiments of Section \ref{sec:audio} utilizes the same backbone with fixed hypotheses heads. More precisely, denoting $k$ the hypothesis index associated with the histogram grid for row $i \in [\![1,N_{\mathrm{rows}}]\!]$ and column $j \in [\![1,N_{\mathrm{cols}}]\!]$ we have, for every $x \in \cX$: $$f_{\theta}^{k}(x) = \left( -\pi+\left(i-\frac{1}{2}\right) \frac{2 \pi}{N_{\mathrm{rows}}}, -\frac{\pi}{2}+\left(j-\frac{1}{2}\right) \frac{\pi}{N_{\mathrm{cols}}}\right),$$ 
where these coordinates correspond, respectively, to the azimuth and the elevation within the ranges $[-\pi,\pi]$ and $[-\frac{\pi}{2}, \frac{\pi}{2}]$. Note that for the 2D Histogram, $K = N_{\mathrm{rows}} N_{\mathrm{cols}}$ with the notations of Table \ref{tab:cvde_vs_kde}. The results of for the 2D Histogram were computed with $N_{\mathrm{rows}} = N_{\mathrm{cols}}$, except when $K = 20$ (in Figure \ref{fig:app-nll-h-ansyn} and Table \ref{tab:quantization_error}) where we set $N_{\mathrm{rows}} = 5$ and $N_{\mathrm{cols}} = 4$. 

\textbf{Training details}. The trainings were conducted using the
AdamW optimizer \cite{loshchilov2018decoupled}, with a batch size of 32, an initial learning rate of $0.05$, and following the scheduling scheme from \citet{vaswani2017attention}. The WTA model was trained using the multi-target version of the
Winner-takes-all loss (Equation 2 and 5 of \citet{letzelter2023resilient}), using confidence weight $\beta = 1$. Note that as the predictions and the targets belong to the unit sphere, the underlying loss $\ell$ used was the spherical distance $\ell(\hat{\by}, \by) = \arccos [\hat{y}^{\top} y]$ where $y, \hat{y} \in \mathbb{S}^2 \subseteq \mathbb{R}^3$.

\textbf{Synthetic perturbation.} To properly evaluate the ability of the baseline methods to predict conditional distributions, we propose a protocol that involves injecting heteroscedastic noise.
\begin{enumerate}
     \item During training, we performed \textit{class-conditioned perturbation} of the target source positions. More precisely, for each of the $n_{\mathrm{c}} = 11$ classes in the dataset
     (\texttt{speech, phone, keyboard, doorslam, laughter, keysDrop, pageturn, drawer, cough, clearthroat, knock}), we randomly perturbed the target position of the sound sources using a distinct standard deviation assigned to each class. These perturbations were drawn from normal distributions on the spherical coordinates, with angular standard deviations (in degrees) in the set $\{5 + 5c\;|\;c \in [\![0,n_{\mathrm{c}} - 1]\!] \}$.
    \item Once trained, the hope is that at inference time, the models have understood the ambiguity in the data, such that they can infer, given a new input audio snippet, the spatial spread of the sound source. 
    \item We evaluate the quality of the estimated distribution with regard to the ground-truth distribution, which is known here.

\end{enumerate}

\textbf{Evaluation details.} For preserving the geometry of the sphere, the likelihoods of \naive\ and \ours\ were computed using isotropic von Mises-Fisher kernels. For data that lives on a 2-dimensional sphere, \textit{i.e.,} for $f_{\theta}^k(x),y \in \mathbb{S}^2 \subseteq \mathbb{R}^3$, such kernels write in the form
$$
K_h(f_{\theta}^k(x),y) = \frac{\kappa}{4 \pi \sinh \kappa} \exp \left(\kappa y^{\top} f_{\theta}^k(x)\right)\,,
$$
where $h > 0$ is the scaling factor associated with the kernel defined as $h = \frac{1}{\sqrt{\kappa}}$. 
The kernel is thus defined with a single concentration parameter $\kappa > 0$.
The volume $V(f_\theta^k(x), K_h) \triangleq \int_{\cY_\theta^k(x)} K_h \left(f_\theta^k(x), \tilde{\by}\right) \mathrm{d} \tilde{\by}$ of the predicted distribution of Voronoi-WTA in each curved cell $k$ was computed noting that 
\begin{equation}\int_{\cY_\theta^k(x)} K_h \left(f_\theta^k(x), \tilde{\by}\right) \mathrm{d} \tilde{\by} \simeq \mathcal{S}(\cY_\theta^k(x)) \mathbb{E}_{Y \sim \mathcal{U}(\cY_\theta^k(x))}[K_h (f_\theta^k(x), Y)]\,,
\label{eq:mc-approx}
\end{equation}
where $\mathcal{S}(\cY_\theta^k(x))$ is the surface of cell $k$. Note that we also have \begin{equation}
    \mathcal{S}(\cY_\theta^k(x)) = 4 \pi \mathbb{E}_{Y \sim \mathcal{U}(\mathbb{S}^2)}[\mathds{1}(Y \in \cY_\theta^k(x))]\,.
    \label{eq:vol-approx}
    \end{equation} 

In practice, the expectations in \eqref{eq:mc-approx} and \eqref{eq:vol-approx} were computed using monte-carlo estimates, \textit{i.e.},
$$ \mathbb{E}_{Y \sim \mathcal{U}(\cY_\theta^k(x))}[K_h (f_\theta^k(x), Y)] \simeq \frac{1}{N^{\prime}} \sum_{y_i \sim \mathcal{U}(\cY_\theta^k(x))} K_h (f_\theta^k(x), y_i)\,,$$
$$\mathbb{E}_{Y \sim \mathcal{U}(\mathbb{S}^2)}[\mathds{1}(Y \in \cY_\theta^k(x))] \simeq \frac{1}{N^{\prime}} \sum_{y_i \sim  \mathcal{U}(\mathbb{S}^2)} \mathds{1}(y_i \in \cY_\theta^k(x))\,,$$
where $N^{\prime}$ is the number of sampled points. The uniform sampling on the unit sphere $\mathcal{U}(\mathbb{S}^2)$ has been performed by sampling the azimuth and elevation angles with $\phi \sim \mathcal{U}([0,2\pi])$ and $\theta \sim \mathrm{arccos}[\mathcal{U}([-1,1])]-\pi/2$ \cite{weisstein2002sphere}. 

The $\mathrm{NLL}$ defined in \eqref{eq:nll} was itself computed using a single sample from the target distribution $\rho_x$ for each input, which can be assumed to be known in the synthetic data perturbation setup (which is as a mixture of Gaussians here; see Figure \ref{fig:qualitative_plot_audio} for an illustration).

For adapting the quantization error to the spherical geometry, \eqref{eq:quantization_error} was generalized with 
\begin{align}
    \distortion(\mathcal{Z}) = \int_\cY \min_{z\in\cZ} \; \mathrm{dist}(y,z)^2 \;\rho_x(\by) \mathrm{d}\by\,,
\end{align}
where $\mathrm{dist}$ is the spherical distance defined as $\mathrm{dist}(y,z) = \arccos [y^{\top} z]$ for $y,z \in\mathbb{S}^2 \subseteq \mathbb{R}^3$.

Our implementation was based on \citet{schymura2021pilot, schymura2021exploiting, letzelter2023resilient, polianskii2022voronoi}. 

\subsubsection{Additional results}
\label{app:additional-results-appendix}

\begin{figure}
    \centering
    \includegraphics[width=\columnwidth]{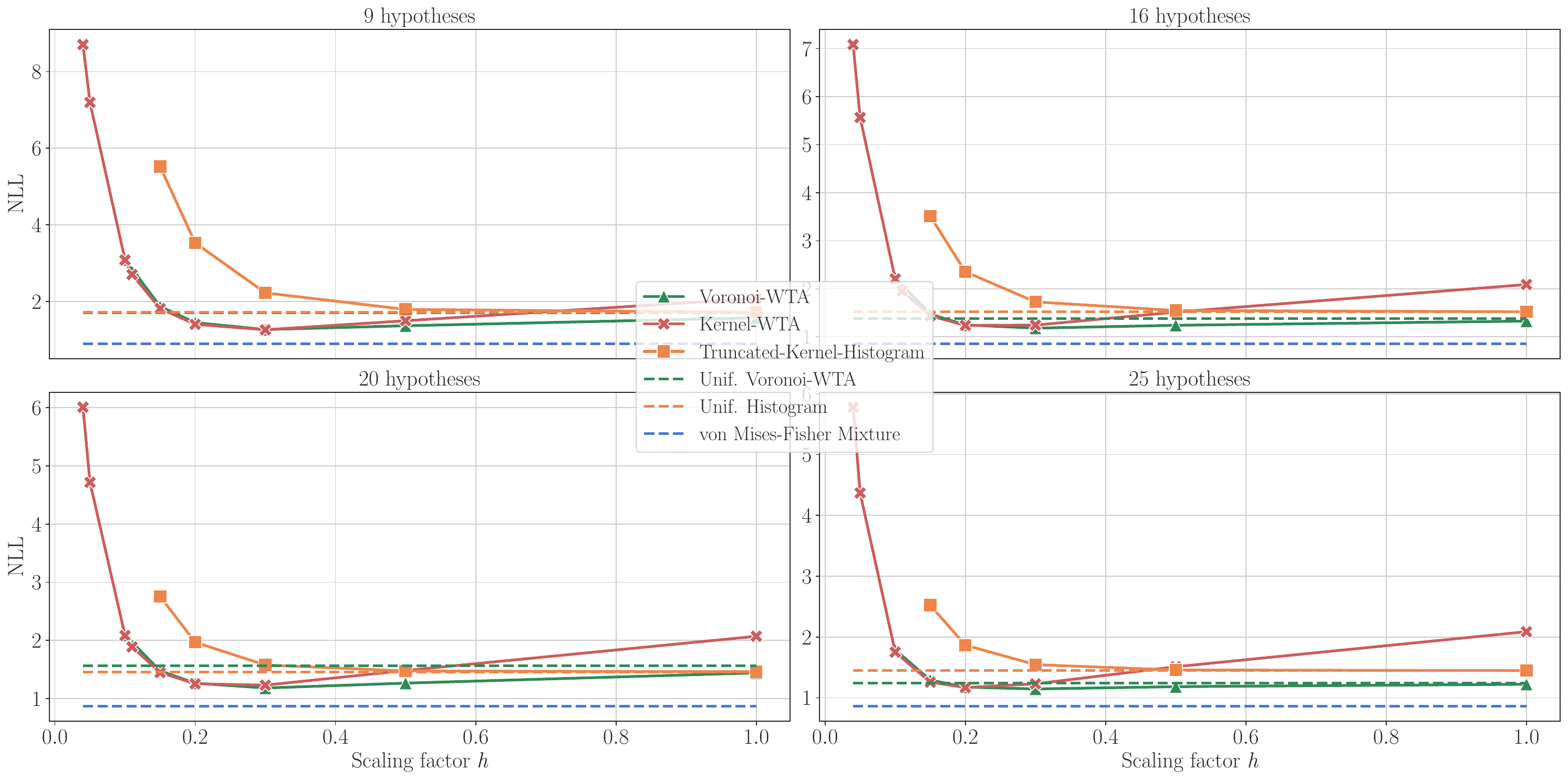}
    \caption{\textbf{$\mathrm{NLL}$ \textit{vs.} $h$ results on ANSYN with 9, 16, 20 and 25 hypotheses.} $\mathrm{NLL}$ on the test set of the spatial audio dataset ANSYN (see Section \ref{sec:audio}) as a function of the scaling factor $h$. `Von Mises mixture' corresponds to a von Mises-Fisher-based mixture density network. The legend is the same as in Figure \ref{fig:NLLvsh}. `Truncated-Histogram' corresponds to the standard Histogram where truncated kernels are placed on the fixed hypotheses, instead of standard uniform ones. `Unif. Histogram' and `Unif. Voronoi WTA' corresponds to uniform truncated kernels, which correspond to the limit of their von Mises-Fisher truncated variant as $h \rightarrow \infty$. As expected, `Kernel-WTA' and `Voronoi-WTA' coincide as $h \rightarrow 0$ (see Section \ref{sec:truncated-kernels}). Here, dashed lines correspond to baselines that are independent of $h$.}
    \label{fig:app-nll-h-ansyn}
\end{figure}

Additional results complementing those in Table \ref{tab:cvde_vs_kde} in the ANSYN audio dataset are presented in Figure \ref{fig:app-nll-h-ansyn} and in Table \ref{tab:quantization_error}, corresponding to $9, 16, 20$ and $25$ hypotheses, respectively. 

In Figure \ref{fig:app-nll-h-ansyn} each subplot illustrates the Negative Log-Likelihood ($\mathrm{NLL}$) on the test set as a function of the scaling factor $h$, across different baselines. The legends in the figures follow the format of Figure \ref{fig:NLLvsh}: `WTA' denotes methods trained using the Winner-takes-all scheme, `Histogram' refers to the Histogram baseline, `Voronoi' denotes the application of truncated kernels based on the `Histogram' or `WTA' hypotheses, and `Kernel-WTA' is defined in Section \ref{sec:Kernel-WTA}. `Unif.' indicates the utilization of uniform kernels instead of von Mises-Fisher kernels, when using truncated kernels estimators. Here, we introduce an additional baseline for further discussion: a mixture density network based on the von Mises-Fisher distribution (M-vMF). It is based on the same training loss as \eqref{eq:mdn}, but considering instead a mixture of von Mises-Fisher distribution $$\hat{\rho}_{\theta}(y | x) = \sum_{k=1}^{K} \pi_k(x) \frac{\kappa_k(x)}{4 \pi \sinh \kappa_k(x)} \exp \left(\kappa_k(x) y^{\top} \mu_k(x)\right),$$ with parameters $\{\pi_k(x), \mu_k(x), \kappa_k(x)\}$. Stability issues were observed in M-vMF training, including numerical overflow in $\mathrm{NLL}$ computation. Those were partially reduced, for instance using the following expression, $\mathrm{log} (\mathrm{sinh}(t)) = t + \mathrm{log}(1-e^{-2t}) - \mathrm{log}(2)$ allowing computation stability for large $t \in \mathbb{R}_{+}^{*}$. 

\textbf{NLL Comparison.} First, it is important to highlight the consistency observed in the results as the scaling factor $h$ approaches zero. In this limit, the NLL values for both Voronoi-WTA and Kernel-WTA tend to coincide, an outcome that aligns with expectations; the impact of truncation diminishes in this scenario, as elaborated in Section \ref{sec:quantitative-results}. Analogously to observations made in Figure \ref{fig:NLLvsh}, the NLL values associated with the Histogram baseline (orange squares) demonstrate the fastest divergence as $h \rightarrow 0$, due to the non-optimal positioning of the hypotheses.

As the scaling factor $h$ increases, a quantitative improvement is observed in the $\mathrm{NLL}$ performance of Voronoi-WTA (green triangles) compared to other baseline methods. Notably, we notice that the Kernel Density Estimation (KDE) approaches, namely Kernel-WTA and Unweighted Kernel-WTA, exhibit a performance decline with increasing $h$, in contrast to the truncated kernel versions (Voronoi-WTA and Voronoi-Histogram, with Indian red color in the plots). This divergence is attributed to the dispersion of probability mass beyond the boundaries of the Voronoi cells in KDE-variants, leading to a loss of local geometric properties, as detailed in Section \ref{sec:Kernel-WTA}. 

Furthermore, as outlined in Section \ref{sec:audio}, note that the performance gap between Kernel-WTA and Voronoi-WTA tends to narrow with a decrease in the number of hypotheses. This is because, in scenarios with fewer hypotheses, the impact of kernel truncation becomes less significant. We see that in those settings, the Voronoi-WTA almost reaches the performance of the M-vMF mixture density network, which slightly outperforms the other estimators in terms of NLL, with the gap getting closer when the number of hypotheses is large. This is promising, as von Mises Fisher has three advantages in this context: 1) it optimizes NLL during training; 2) the audio dataset targets are perturbed with synthetic Gaussian angular noise, which is similar to the von Mises Fisher kernel; 3) the M-vMF has more parameters than the Voronoi-WTA method, as it allows for a variable concentration parameter in each cell. For a fairer comparison, experiments on real-world data without synthetic perturbations were performed in Appendix \ref{app:uci-datasets}, showing that Voronoi-WTA can even outperform MDN-based methods in terms of NLL in some settings. 

\textbf{Quantization Error comparison.} Quantization error results are provided in Table \ref{tab:quantization_error}, are consistent with the results of Section \ref{sec:audio}. First, we see that for all methods, the quantization improves with $K$, which is consistent. Secondly, we see an advantage in the quantization error between the WTA-based and the M-vMF methods. This is consistent with the third line of Figure \ref{fig:quantitative}, which shows that WTA tends to outperform MDN in terms of quantization error.

Visualizations of the WTA predictions on audio data are also provided in Figure \ref{fig:qualitative_plot_audio}.

\begin{table}
\caption{\textbf{Quantization Error comparison.}
Quantization Error ($\times 10^{2}$) for the compared estimators on the spatial audio dataset ANSYN (see Section 6.4) with spherical underlying distances (expressed in radians). M-vMF corresponds to a von Mises-Fisher-based mixture density network.}
\centering 
\begin{tabular}{lcccccccccccc}
\toprule \multicolumn{1}{c}{} & \multicolumn{4}{c}{$K$} \\
 \cmidrule(lr){2-5}
 Estimator & 9 & 16 & 20 & 25\\
 \midrule
 V-WTA & \textbf{0.42} & \textbf{0.26} & \textbf{0.27} & \textbf{0.17} \\ 
 K-WTA & \textbf{0.42} & \textbf{0.26} & \textbf{0.27} & \textbf{0.17} \\
 M-vMF & 0.62 & 0.40 & 0.34 & 0.29\\
Hist & 1.23 & 0.72 & 0.55 & 0.47\\
\bottomrule
\end{tabular}
\label{tab:quantization_error}
\end{table}

\begin{figure}
    \centering
    \includegraphics[width=\columnwidth]{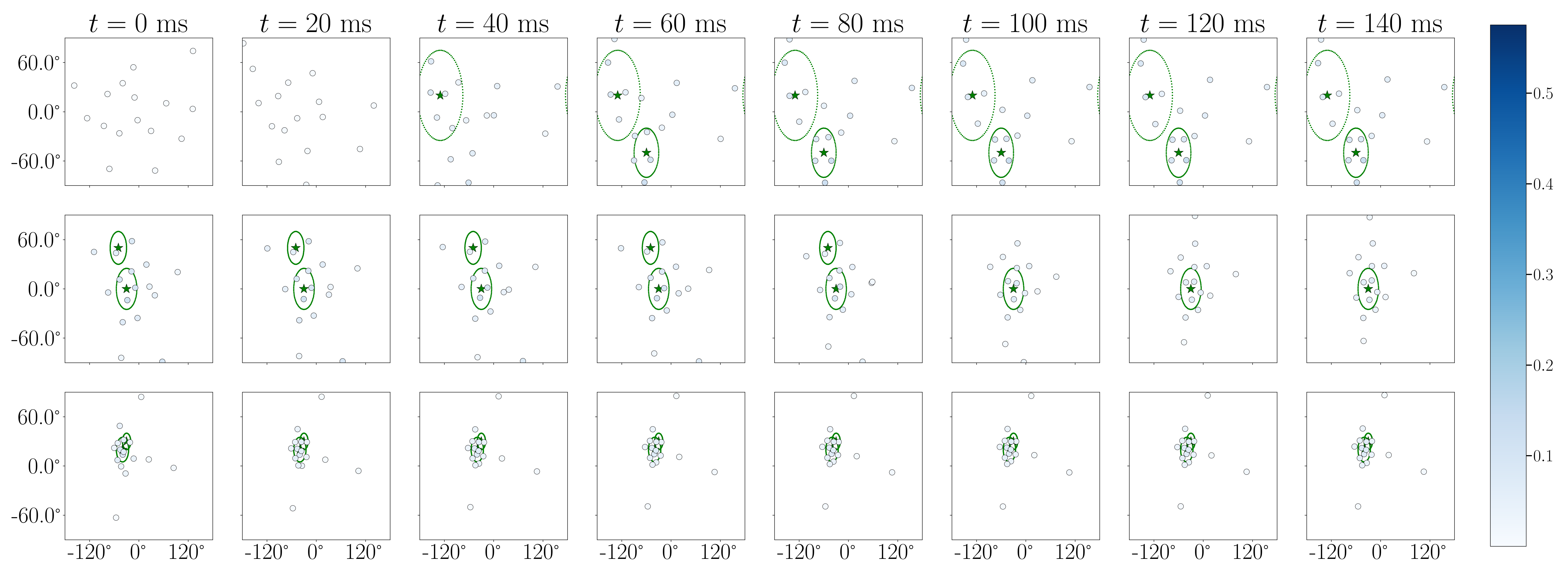}
    \caption{\textbf{Quantifying Uncertainty in Spatial Audio with WTA Learners.} Results for audio clips from the test set of the ANSYN dataset. Each row corresponds to distinct recordings, while each column represents different time frames. The subplot's axes correspond to azimuth and elevation in degrees. The target positions of sound sources, marked by green stars, are within ellipses showing their theoretical spatial dispersion. The WTA model's 16 predictions are indicated by blue shaded circles, and the score confidence is indicated by the shade intensity (the legend is given by the color bar on the \textit{right}). 
    We can see that the prediction dispersion follows closely the theoretical dispersion of the sources. This demonstrates that the model successfully grasps input-dependent uncertainty.} 
    \label{fig:qualitative_plot_audio}
\end{figure}

\subsection{Computation details}

In this research, we utilized the \href{https://www.python.org/}{Python} programming language, along with the \href{https://pytorch.org/Pytorch}{Pytorch} \cite{paszke2019pytorch} deep learning framework. We also employed the \href{https://github.com/facebookresearch/hydra}{Hydra} and \href{https://pypi.org/project/mlflow/}{MLFlow} libraries for experimental purposes. Our coding was inspired by several previous works \cite{ adavanne2019localization, makansi2019overcoming, schymura2021pilot, polianskii2022voronoi, han2022card, letzelter2023resilient}. The training of our neural networks was conducted on NVIDIA A100 GPUs.

\end{document}